\documentclass{article}
\usepackage{enumitem}
 \usepackage[preprint]{neurips_2025}

\usepackage[utf8]{inputenc} 
\usepackage[T1]{fontenc}    
\usepackage{hyperref}       
\usepackage{url}            
\usepackage{booktabs}       
\usepackage{amsfonts}       
\usepackage{nicefrac}       
\usepackage{microtype}      
\usepackage{xcolor}
\usepackage{tabularx}
\usepackage{amsthm}
\usepackage{float}
\usepackage[section]{placeins}
\usepackage[ruled,vlined,linesnumbered]{algorithm2e}
\SetKwInput{KwIn}{Input}
\SetKwInput{KwOut}{Output}

\DontPrintSemicolon
\usepackage{listings}
\usepackage{multirow}
\usepackage{amsmath,amssymb,mathtools}
\usepackage{natbib} 
\usepackage{bm}
\usepackage{hyperref}
\usepackage{graphicx}
\usepackage{subcaption}
\usepackage{amsthm}

\newcommand{\R}{\mathbb{R}}

\providecommand{\keywords}[1]{%
  \par\nopagebreak[4]\vspace{0.15ex}%
  \noindent{\footnotesize\textbf{Keywords: }%
  \begingroup
  \def\and{ \ensuremath{\cdot} }%
  \normalfont #1%
  \endgroup}%
  \par\vspace{0.1ex}
}

\newtheorem{theorem}{Theorem}[section]
\newtheorem{proposition}[theorem]{Proposition}
\newtheorem{lemma}[theorem]{Lemma}
\theoremstyle{remark}
\newtheorem{remark}[theorem]{Remark}

\title{Sinkhorn-Drifting Generative Models}

\author{
  \textbf{Ping He$^{1}$ \quad Om Khangaonkar$^{2}$}\\[0.55ex]
  \textbf{Hamed Pirsiavash$^{2}$ \quad Yikun Bai$^{3,\ddagger}$ \quad Soheil Kolouri$^{1,\ddagger}$}\\[0.8ex]
  {\normalfont $^{1}$Department of Computer Science, Vanderbilt University}\\[-0.15ex]
  {\normalfont $^{2}$Department of Computer Science, University of California, Davis}\\[-0.15ex]
  {\normalfont $^{3}$Department of Computer Science, Purdue University}\\[-0.15ex]
  {\normalfont \texttt{ping.he@vanderbilt.edu, omkhanga@ucdavis.edu, hpirsiav@ucdavis.edu}}\\[-0.15ex]
  {\normalfont \texttt{bai195@purdue.edu, soheil.kolouri@vanderbilt.edu}}
}

\makeatletter
\renewcommand{\@maketitle}{%
  \vbox{%
    \hsize\textwidth
    \linewidth\hsize
    \vskip 0.08in
    \@toptitlebar
    \centering
    {\LARGE\bf \@title\par}
    \@bottomtitlebar
    \vskip -0.04in
    \begin{tabular}[t]{c}\bf\rule{\z@}{14\p@}\@author\end{tabular}%
    \vskip 0.20in \@minus 0.06in
  }%
}
\renewcommand{\@notice}{}
\makeatother

\renewenvironment{abstract}%
{%
  \vskip 0.015in%
  \centerline{\large\bf Abstract}%
  \vspace{0.05ex}%
  \begin{quote}%
}%
{%
  \par%
  \end{quote}%
  \vskip 0.35ex%
}
\begin{document}

\maketitle
\enlargethispage{3\baselineskip}
\begingroup
\renewcommand{\thefootnote}{\fnsymbol{footnote}}
\setcounter{footnote}{2}
\footnotetext{Corresponding Author.}
\endgroup
\begin{figure}[H]
\centering
\vspace{-0.18in}
\includegraphics[width=0.95\linewidth]{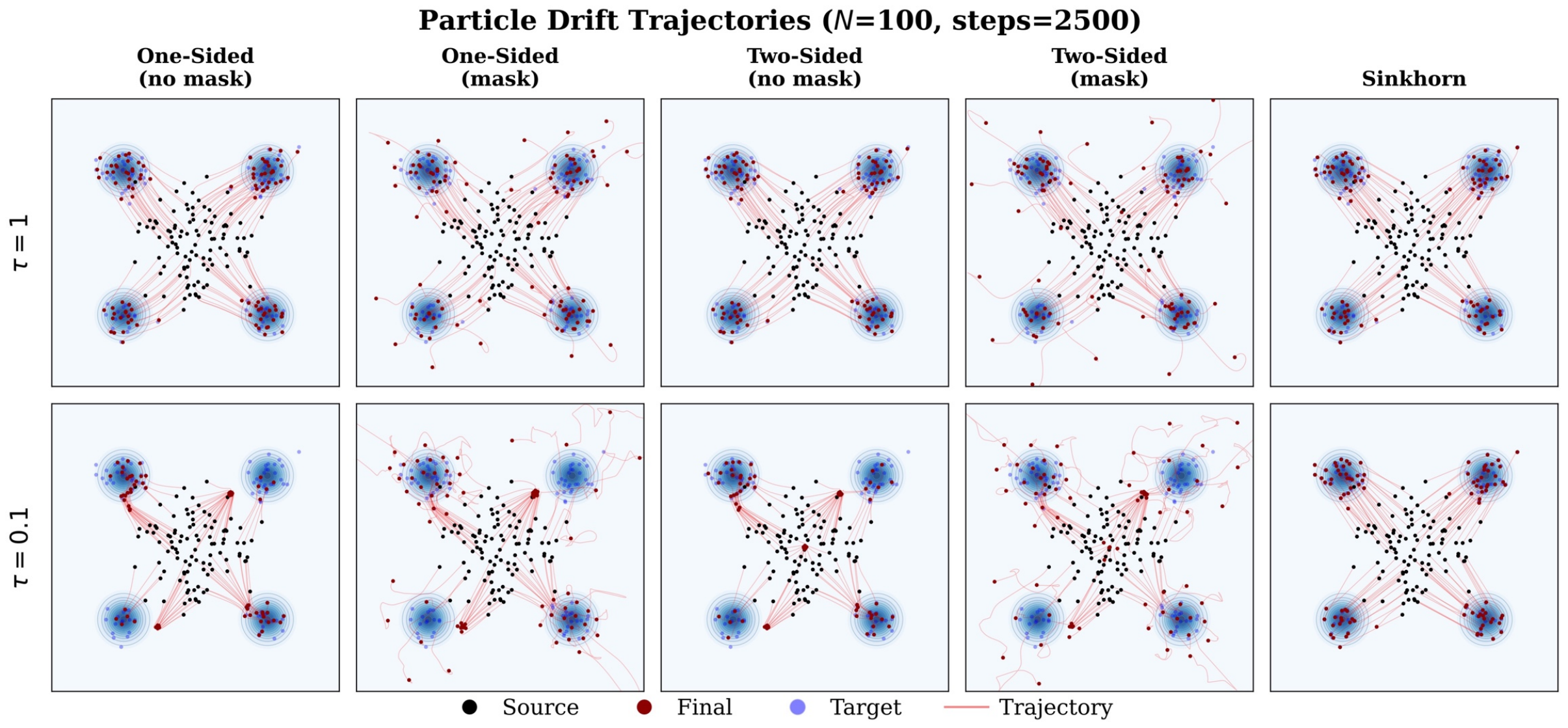}
\vspace{-0.04in}
\setlength{\abovecaptionskip}{5pt}
\setlength{\belowcaptionskip}{4pt}
\caption{Drift trajectories across different values of $\tau$ under one-sided, two-sided, and Sinkhorn normalization, with and without self-distance masking. The Sinkhorn trajectories are generated with a fixed number of iterations, $30$.}
\vspace{-0.13in}
\label{fig:drift_trajectories_eps_sweep}
\end{figure}

\begin{abstract}
We establish a theoretical link between the recently proposed “drifting” generative dynamics and gradient flows induced by the Sinkhorn divergence. In a particle discretization, the drift field admits a cross-minus-self decomposition: an attractive term toward the target distribution and a repulsive/self-correction term toward the current model, both expressed via one-sided normalized Gibbs kernels. We show that Sinkhorn divergence yields an analogous cross-minus-self structure, but with each term defined by entropic optimal-transport couplings obtained through two-sided Sinkhorn scaling (i.e., enforcing both marginals). This provides a precise sense in which drifting acts as a surrogate for a Sinkhorn-divergence gradient flow, interpolating between one-sided normalization and full two-sided Sinkhorn scaling. Crucially, this connection resolves an identifiability gap in prior drifting formulations: leveraging the definiteness of the Sinkhorn divergence, we show that zero drift (equilibrium of the dynamics) implies that the model and target measures match. Experiments show that Sinkhorn drifting reduces sensitivity to kernel temperature and improves one-step generative quality, trading off additional training time for a more stable optimization, without altering the inference procedure used by drift methods. These theoretical gains translate to strong low-temperature improvements in practice: on FFHQ-ALAE at the lowest temperature setting we evaluate, Sinkhorn drifting reduces mean FID from $187.7$ to $37.1$ and mean latent EMD from $453.3$ to $144.4$, while on MNIST it preserves full class coverage across the temperature sweep. Project page: \url{https://mint-vu.github.io/SinkhornDrifting/}.

\end{abstract}

\section{Introduction}
Drifting Models~\cite{deng2026generative} propose a one-step generative modeling paradigm that, during training, minimizes the ``drift'' between the pushforward of a source distribution and the target data distribution. They define a cross-minus-self (attractive–repulsive) drift that pulls samples toward the data distribution while pushing them away from the current model, and leverage it as a stop-gradient regression signal during training to steer the generator in the induced direction; at test time, generation is a single forward pass. On ImageNet $256\times256$, drift models \cite{deng2026generative} achieve state-of-the-art performance among one-step methods, with FID competitive with multi-step diffusion and flow-based generators. Interestingly, the construction in~\cite{deng2026generative} is closely related to earlier generative modeling approaches~\cite{li2017mmd,unterthiner2018coulomb} based on maximum-mean discrepancy (MMD), which also combine attractive and repulsive forces to shape the generator. However, it is important to note that unlike standard MMD formulations, drifting models introduce an additional normalization that breaks the interpretation as a gradient flow of MMD, hence making them different from MMD gradient flow approaches.

Despite this strong empirical performance, a fundamental theoretical question remains open. The drifting field between the source and target distributions, $p$ and $q$, denoted $V_{q,p}$, is constructed so that $p=q$ implies $V_{q,p}\equiv 0$. However, the converse---namely, that $V_{q,p}\equiv 0$ implies the generated and target distributions coincide---is not established in general. The identifiability argument provided by \cite{deng2026generative} relies on a non-degeneracy assumption (linear independence of bilinear interaction vectors) that may not hold for all configurations of kernels and distributions. This leaves a gap between the training objective, which minimizes $\|V\|^2$, and the generative modeling goal of achieving $q = p$.

In this paper, we make an observation that connects the Drifting Model dynamics to the theory of Sinkhorn Divergences~\cite{feydy2019interpolating}, a well-studied family of loss functions interpolating between optimal transport (OT) and maximum mean discrepancy (MMD), which have been used for generative modeling in prior work \cite{genevay18a,salimans2018improving}. Specifically, we show that at the particle level, both the drifting field and the gradient flow of the Sinkhorn divergence share the same algebraic structure: a \emph{cross-minus-self} barycentric projection $V(X) = P_{\mathrm{cross}}Y - P_{\mathrm{self}}X$, where the coupling matrices $P$ are row-stochastic. The two formulations differ only in how these couplings are constructed: Drifting Models use (approximate) row-normalization of a Gibbs kernel, while the Sinkhorn divergence uses entropic optimal transport plans obtained by full two-sided Sinkhorn scaling.

This connection offers several insights. First, it situates Drifting Models within a well-understood variational framework, revealing that they implement an approximate version of a Sinkhorn-divergence gradient flow. Second, it explains \emph{why} the heuristic and empirically observed two-sided normalization used in the actual Drift implementation \cite{deng2026generative} (a geometric mean of row- and column-softmax) improves over pure row-normalization: it moves the coupling closer to the doubly stochastic Sinkhorn solution. Third, it clarifies the identifiability gap: the Sinkhorn divergence satisfies $S_\tau(\alpha,\beta) = 0 \Leftrightarrow \alpha = \beta$ rigorously, and this guarantee traces directly to the global marginal constraints enforced by two-sided scaling, constraints that one-sided normalization lacks. The additional computational overhead of Sinkhorn normalization during training buys us temperature-stable, theoretically grounded drift fields with fewer ad hoc engineering tricks required for training, while keeping inference time \emph{exactly} the same.


\section{Background and Notations}

\subsection{Drift Method}
\paragraph{Drift Velocity Field.}
Let $p$ denote a fixed target distribution on $\mathbb{R}^d$, and let $q$ denote the current model distribution. The Drift is constructed as a velocity field $V_{q,p}:\mathbb{R}^d\to \mathbb{R}^d$ of the form
\begin{equation}
V_{q,p}(x) = V_p^+(x) - V_q^-(x),
\label{eq:drift_field_def}
\end{equation}
where each term is a kernel-weighted (and normalized) average of displacement vectors. Concretely, given a positive kernel $k:\mathbb{R}^d\times\mathbb{R}^d\to \R_+$ and the normalizers
\begin{equation}
Z_p(x):=\int k(x,y)\,dp(y),\qquad Z_q(x):=\int k(x,y)\,dq(y),\nonumber
\end{equation}
the Drift update is
\begin{equation}
V_p^+(x) := \frac{1}{Z_p(x)}\int k(x,y)(y-x)\,dp(y),\qquad
V_q^-(x) := \frac{1}{Z_q(x)}\int k(x,y)(y-x)\,dq(y).
\label{eq:drift_terms}
\end{equation}
Intuitively, $V_p^+(x)$ attracts $x$ toward the target $p$ (a kernel barycenter step), while $V_q^-(x)$ subtracts an analogous attraction toward the current model $q$, producing a repulsive (self-correction) effect.

We use the Gibbs kernel $k(x,y)=e^{-\frac{C(x,y)}{\tau}}$\footnote{The kernel used in \cite{deng2026generative} is $e^{-\frac{\|x-y\|}{\tau}}$; here we instead use the classical Gaussian kernel.} where $C(x,y)$ is a cost function between $x$ and $y$; by default, we set $C(x,y)=\frac{1}{2}\|x-y\|^2$ and use constant $\tau>0$.

In the discrete setting, let $p=\frac{1}{n}\sum_{j=1}^n\delta_{y_j}$ and $q=\frac{1}{n}\sum_{i=1}^n\delta_{x_i}$ be empirical measures with samples $Y=\{y_j\}_{j=1}^n$ and $X=\{x_i\}_{i=1}^n$. Then
\begin{align}
&V_{q,p}(x_i)=\sum_{j=1}^n P^{\mathrm{drift}}_{XY}[i,j]y_j-
\sum_{j=1}^n P^{\mathrm{drift}}_{XX}[i,j]x_j\label{eq:drift_discrete}\\
&\text{where }P_{XY}^{\mathrm{drift}}[i,j]:=\frac{k(x_i,y_j)}{\sum_{l=1}^n k(x_i,y_l)},\quad P_{XX}^{\mathrm{drift}}[i,j]:=\frac{k(x_i,x_j)}{\sum_{l=1}^n k(x_i,x_l)}\nonumber
\end{align}

\paragraph{Drift generative model.}
We define a one-step generative model:
\begin{align}
f_\theta: \mathbb{R}^d\to \mathbb{R}^d, \qquad \epsilon\mapsto x_\theta:=f_\theta(\epsilon) \label{eq:f_theta}
\end{align}
where $\epsilon\sim p_{\epsilon}$ for some prior distribution $p_{\epsilon}$.
The drift flow is the probability path $q_\theta:=(f_\theta)_\# p_\epsilon$  generated by $V_{q_\theta,p}$:
\begin{align}
  \dot{x_{\theta}}=V_{q_\theta,p}(x_\theta),\forall x_\theta=f_\theta(\epsilon),\label{eq:drift_flow}  
\end{align}
which is approximated by its forward Euler process in time discretization scheme:
\begin{align}
x_{\theta_{k+1}}=x_{\theta_{k}}+V_{q_\theta,p}(x_{\theta_k}).\label{eq:drift_flow_discrete}  
\end{align}
The training loss is then set to be
\begin{align}
\mathcal{L}^{\text{drift}}:=\frac{1}{2}\mathbb{E}_{\epsilon}\Big[\|f_\theta(\epsilon)-\text{sg}\big(f_\theta(\epsilon)+V_{q_\theta,p_{\text{data}}}(f_\theta(\epsilon))\big)\|^2\Big],\label{eq:drift_loss}
\end{align}
where $\text{sg}$ is the stop-gradient operator, and the optimization scheme is based on gradient descent:
\begin{align}
  \theta \gets -\eta\nabla_\theta \mathcal{L}^{\text{drift}},\nabla_\theta \mathcal{L}^{\text{drift}}= -\mathbb{E}_{\epsilon}[J_f(\theta,\epsilon)^\top V_{q_\theta,p_{\text{data}}}(f_\theta(\epsilon))],
  \label{eq:drift_gradient}
\end{align}
where $\eta$ is the learning rate. It can be verified that the above parameter update \eqref{eq:drift_gradient} induces the drift flow \eqref{eq:drift_flow_discrete}. We refer to Appendix~\ref{sec:drift_model} for details.

\subsection{Entropic OT and Sinkhorn divergence}\label{sec:entropic_ot_sinkhorn}
Let $C:\mathbb{R}^d\times\mathbb{R}^d\to \R_+$ be a measurable cost mentioned in the above section, and choose probability measures $\alpha,\beta\in\mathcal{P}(\mathbb{R}^d)$. Furthermore, we assume $\alpha,\beta\in \mathcal{P}_C(\mathbb{R}^d):=\{p\in\mathcal{P}(\mathbb{R}^d):\int c(x,0)dp(x)<\infty\}$. 
The entropic OT is defined as
\begin{equation}
\mathrm{OT}_{\tau}(\alpha,\beta)
:=\min_{\pi\in\Pi(\alpha,\beta)}\int C(x,y)d\pi(x,y)
 +\tau D_{KL}\!\left(\pi\,\middle\|\,\alpha\otimes\beta\right),
\label{eq:entropic_ot}
\end{equation}
where $\alpha\otimes\beta$ is the outer-product measure, $\Pi(\alpha,\beta)$ denotes the set of couplings with marginals $\alpha$ and $\beta$, and $D_{KL}(\cdot~||~\cdot)$
is the Kullback--Leibler (KL) divergence.
The Sinkhorn Divergence \cite{genevay18a, feydy2019interpolating} is the debiased functional
\begin{equation}
S_{\tau}(\alpha,\beta)
:=
\mathrm{OT}_{\tau}(\alpha,\beta)
-\frac12 \mathrm{OT}_{\tau}(\alpha,\alpha)
-\frac12 \mathrm{OT}_{\tau}(\beta,\beta),
\label{eq:sinkhorn_div}
\end{equation}
interpolating OT (as $\tau\to 0$) and an MMD/energy-type geometry (as $\tau\to\infty$).


\subsection{Sinkhorn algorithm}
Let $\alpha=\sum_{i=1}^n\mathrm{p}^\alpha_i\delta_{x_i},\beta=\sum_{i=1}^m\mathrm{p}^\beta_i\delta_{y_i}$ be discrete probability measures on $\mathbb{R}^d$. The entropic OT problem \eqref{eq:entropic_ot} becomes: 
$$
OT_\tau(\alpha,\beta)=\min_{\pi\in\Pi(\mathrm{p}^\alpha,\mathrm{p}^\beta)}\sum_{i,j=1}^{n,m}C(x_i,y_j)\pi_{i,j}+\tau D_{KL}(\pi\parallel \mathrm{p}^\alpha\otimes\mathrm{p}^\beta)
$$
where $\Pi(\mathrm{p}^\alpha,\mathrm{p}^\beta)=\{\pi\in\mathbb{R}^{n\times m}:\pi \bold{1}_m=\mathrm{p}^\alpha, \pi^\top \bold{1}_n=\mathrm{p}^\beta\}, D_{KL}(\pi\parallel \mathrm{p}^\alpha\otimes \mathrm{p}^\beta):=\sum_{i,j}\log(\frac{\pi_{i,j}}{\mathrm{p}^\alpha_i\mathrm{p}^\beta_j})\pi_{i,j}$. The discrete entropic OT problem admits an efficient solution via the Sinkhorn--Knopp
(iterative proportional fitting) algorithm, which can be interpreted as alternating
Bregman projections onto the marginal constraints \cite{cuturi2013sinkhorn}.
In matrix form, starting from the Gibbs kernel
$K := \exp(-C/\tau)$, the algorithm alternates row and column rescalings:
\begin{equation}
\begin{cases}
\pi^{(0)} \gets K,\\
\pi^{(\ell)} \gets \mathrm{diag}\Big(\dfrac{\mathrm p^\alpha}{\pi^{(\ell-1)}\mathbf 1_m}\Big)\pi^{(\ell-1)}, & \text{if $\ell$ is odd},\\
\pi^{(\ell)} \gets \pi^{(\ell-1)}\mathrm{diag}\Big(\dfrac{\mathrm p^\beta}{(\pi^{(\ell-1)})^\top \mathbf 1_n}\Big), & \text{if $\ell$ is even}.
\end{cases}
\label{eq:sinkhorn}
\end{equation}
where the divisions are elementwise.
By \cite{cuturi2013sinkhorn,knight2008sinkhorn,thibault2021overrelaxed}, the Sinkhorn algorithm converges in a number of iterations that typically scales polylogarithmically with the problem size and inversely with the regularization strength (the precise rate depends on the assumptions).

\subsection{Wasserstein Gradient flows.}
Let $\mathcal{P}(\mathbb{R}^d)$ denote the set of all probability measures defined in $\mathbb{R}^d$. A convenient way to describe dynamics on distributions over $\mathbb{R}^d$ is via a \emph{probability path}
(or curve of measures) $(q_t)_{t\in[0,1]}$, where each $q_t$ is a probability measure on $\mathbb{R}^d$.

Given an energy functional $\mathcal{F}(q)$, the \emph{Wasserstein gradient flow} of $\mathcal{F}$ is
the steepest-descent evolution with respect to Wasserstein geometry: 

\begin{equation}
\partial_t q_t +\nabla\cdot\left(q_t v_t\right)=0, v_t =-\nabla \frac{\delta \mathcal{F}}{\delta q}(q_t)
\label{eq:wgf_formal}
\end{equation}
where $\frac{\delta \mathcal{F}}{\delta q}$ denotes the first variation of $\mathcal{F}$, and $\nabla\cdot$ denotes the divergence. For discrete  $q_t:=q_{X_t}:=\sum_{i=1}^n\mathrm{q}_i\delta_{x_t^i}$, \eqref{eq:wgf_formal} induces the following particle dynamics:
\begin{equation}
\dot x_t^i = v_t(x_t^i)
= -\frac{1}{q_j}\nabla_{x_t^i} \big(\mathcal{F}(q_{X_t})\big),\qquad j=1,\ldots,n.
\label{eq:wgf_particle}
\end{equation}
A forward Euler discretization yields
\begin{equation}
x_{k+1}^i = x_k^i-\eta\,\frac{1}{q_i}\nabla_{x^i}\big(\mathcal{F}(q_{X_k})\big),\qquad i=1,\ldots,n.
\label{eq:wgf_step}
\end{equation}
where $x_{k}:=x_{t_{k}}$ and $t_{k}=\frac{k}{T}$ is the time discretization (w.r.t. learning rate, $\frac{1}{T}$).

\section{Wasserstein Gradient Flow of Sinkhorn-Divergence}
\subsection{Sinkhorn Divergence Flow Model}
We consider the quadratic loss $c(x,y)=\frac{1}{2}\|x-y\|^2$ and let $\mathcal{P}_2(\mathbb{R}^d)=\mathcal{P}_{c}(\mathbb{R}^d):=\{p\in\mathcal{P}(\mathbb{R}^d):\int |x|^2 dp(x)<\infty\}$. We suppose $\alpha,\beta\in\mathcal{P}_2(\mathbb{R}^d)$.  For convenience, we denote $OT_\tau^l(\alpha,\beta)=\sum_{i,j=1}^n c(x_i,y_j)\,(\pi_{\alpha,\beta})_{i,j}^l$, where $\pi^l_{\alpha,\beta}$ is the transport plan obtained from the $l$-th iteration of the Sinkhorn algorithm \eqref{eq:sinkhorn}. We then define
$$S_{\tau}^l(\alpha,\beta)=OT_\tau^l(\alpha,\beta)-\frac{1}{2}OT_\tau^l(\alpha,\alpha)-\frac{1}{2}OT_\tau^l(\beta,\beta).$$
This is the corresponding ``Sinkhorn divergence'' where $S_\tau(\alpha,\beta)=S_\tau^\infty(\alpha,\beta)$ and $\pi^\infty_{\alpha,\beta}$ is the solution of $S_\tau(\alpha,\beta)$. Let $p:=p_{\mathrm{data}}\in\mathcal{P}_2(\mathbb{R}^d)$ be a fixed probability measure, and define the functional $\mathcal{F}(q):=S_\tau(p,q)$ with initial $q_0\in\mathcal{P}_2(\mathbb{R}^d)$.  
We consider the Wasserstein gradient flow \eqref{eq:wgf_formal}. 

\begin{proposition}\label{pro:sinkhorn_wgf}
Under the finite-sample approximation $\hat{p}_{\mathrm{data}}=\sum_{j=1}^n\frac{1}{n}\delta_{y^j}$ and $\hat{q}:=\hat{q}_X=\sum_{i=1}^n\frac{1}{n}\delta_{x^i}$, the above Wasserstein gradient flow becomes:
\begin{align}
  \dot{x}^i=:V^\infty_{\hat{q},\hat{p}}(x^i)=\sum_{j=1}^n (n\pi_{XY}^\infty)_{ij}y^j-\sum_{j=1}^n (n\pi_{XX}^\infty)_{ij}x^j,\quad \forall i;\label{eq:sinkhorn_flow}
\end{align} 
and the Euler forward step \eqref{eq:wgf_step} becomes:
$$x^i_{k+1}\gets x^i_k+\eta V_{q_X,p}^\infty(x^i_k),\qquad\forall i,$$ 
where $\pi_{XY}^\infty=\pi^\infty_{q_X,p}$ and $\pi_{XX}^\infty$ is defined similarly.
\end{proposition}
\begin{remark}
In the general probability distribution case, $p\in \mathcal{P}_2(\mathbb{R}^d)$, the Wasserstein gradient flow \eqref{eq:wgf_formal} induces:
\begin{align}
  \dot{x}=\int (y-x)d\pi^\infty_{q,p}(y|x)-\int (y-x)d\pi^\infty_{q,q}(y|x):=V_{q,p}^\infty(x),\label{eq:sinkhorn_drif_general}  
\end{align}
which can be treated as the general version of $V_{\hat{q},\hat{p}}$, where $\pi^\infty(\cdot|x)$ is the conditional measure based on $\pi^\infty$.  Appendix~\ref{sec:sinkhorn_wgf} provides the formal proofs.
\end{remark}





\begin{algorithm}[t]
\caption{Sinkhorn Drifting Field}
\label{alg:split_sinkhorn_drift}
\begin{lstlisting}[language=Python, mathescape=true, escapeinside={(*}{*)}, basicstyle=\footnotesize\ttfamily, aboveskip=2pt, belowskip=2pt, lineskip=-0.2ex]
def sinkhorn_drift(X, Y_pos, Y_neg, (*$\tau$*), T):
  # X:     [N, d]   generated features
  # Y_pos: [N+, d]  positive samples 
  # Y_neg: [N-, d]  negative samples 
  # temperature (*$\tau$*) > 0,
  # Sinkhorn iterations T (an Odd number) (T=1 recovers Deng et al. [2])
  
  # pairwise distances
  D_pos = cdist$^2$(X, Y_pos)          # [N, N+]
  D_neg = cdist$^2$(X, Y_neg)          # [N, N-]

  # Gibbs kernels (logits)
  L_pos = exp(-D_pos / (*$\tau$*))
  L_neg = exp(-D_neg / (*$\tau$*))

  # uniform marginals (or user-specified)
  r = (*$\mathbf{1}_N$*) / N
  c_pos = (*$\mathbf{1}_{N_+}$*) / N_pos
  c_neg = (*$\mathbf{1}_{N_-}$*) / N_neg

  # split Sinkhorn couplings
  (*$\Pi$*)_pos = Sinkhorn(L_pos, r, c_pos, T)   # [N, N+]
  (*$\Pi$*)_neg = Sinkhorn(L_neg, r, c_neg, T)   # [N, N-]

  # row-normalize the Sinkhorn couplings
  P_pos = RowNormalize((*$\Pi$*)_pos)
  P_neg = RowNormalize((*$\Pi$*)_neg)

  # cross-minus-self drift
  V = P_pos @ Y_pos - P_neg @ Y_neg    # [N, d]
  return V
\end{lstlisting}
\end{algorithm}








\subsection{Relation to Drift Field Method}
 Similar to $V^\infty_{q,p}$ defined in \eqref{eq:sinkhorn_flow},  for each $l\in\mathbb{N}$, we introduce the flow: 
$$\dot{x}^i=V^l_{\hat{q}_X,\hat{p}}(x^i)=\sum_{j=1}^n(n\pi_{XY}^l)_{ij}(y_j-x_i)-\sum_{j=1}^n(n\pi_{XX}^l)(x_j-x_i).$$
Thus we have: 
\begin{proposition}
When $l=1$,
the Euler discretization
$$x_{k+1}^i = x_k^i + \eta\, V^l_{\hat{q}_X,\hat{p}}(x^i),\qquad i\in[1:n]$$
coincides with the Drift dynamic \eqref{eq:drift_particles_vfield} as proposed in Deng et al. \cite{deng2026generative}.
\end{proposition}
\begin{proof}
Since $l=1$, we have $P_{XY}^{\text{drift}}=n\pi^l_{XY}$ and $P_{XX}^{\text{drift}}=n\pi^l_{XX}$. Thus,
$$
V^l_{\hat{q},\hat{p}}(x)=V^{\text{drift}}_{\hat{q},\hat{p}}(x),
$$
where $V^{\text{drift}}_{\hat{q},\hat{p}}$ is defined in \eqref{eq:drift_particles_vfield}. We immediately obtain the conclusion.
\end{proof}

\subsection{Training Loss and Algorithm.}
In the generative model setting, we suppose $x=f_\theta(\epsilon),\epsilon\sim p_{\epsilon}$ where $p_\epsilon$ is some prior distribution (e.g. Gaussian), $f_\theta:\mathbb{R}^d\to \mathbb{R}^d$ is a generator. 
By Theorem 1 in \cite{feydy2019interpolating}, under some regular conditions of $p,q$, we have 
\begin{equation}
S_\tau(p,q) = 0 \quad\Longleftrightarrow\quad p = q.
\label{eq:sinkhorn_posdef}
\end{equation}
\begin{remark}
$S_\tau(\alpha,\beta)=0$ iff $\alpha=\beta$ under some regular conditions. 
The proof proceeds via the strict convexity of the Sinkhorn negentropy $F_\tau(\alpha):=-\tfrac12\mathrm{OT}_\tau(\alpha,\alpha)$, which yields a Bregman divergence (the ``Hausdorff divergence'' $H_\tau$) satisfying $0 \leq H_\tau(\alpha,\beta) \leq S_\tau(\alpha,\beta)$. Since $H_\tau(\alpha,\beta)=0$ implies $\alpha=\beta$ by strict convexity, so does $S_\tau(\alpha,\beta)=0$.
\end{remark}

In this case, we have the following: 
\begin{proposition}\label{pro:sinkhorn_identity}
Fix $l\in \mathbb{N}\cup\{+\infty\}$. If the two empirical measures coincide, i.e., $\hat p=\hat q$, then the level-$l$ drift field vanishes:
\begin{align}
V_{\hat q,\hat p}^l(x^i)=0,\qquad \forall i\in[1\!:\!n].\nonumber
\end{align}
\end{proposition}

When $l=\infty$, this admits a gradient-flow interpretation: since $S_\tau(\hat p,\hat q)$ is minimized at $\hat p=\hat q$ (and equals $0$), the associated gradient-flow velocity is zero, hence $V^{\infty}_{\hat q,\hat p}(x^i)=0$.

\begin{remark}
The same ``$p=q\Rightarrow V=0$'' conclusion holds beyond the empirical case. For general probability measures, one can define the Sinkhorn velocity (for $l=\infty$, see \eqref{eq:sinkhorn_drif_general}) and show that the drift vanishes whenever $p=q$; see Appendix~\ref{sec:sinkhorn_model}.
\end{remark}

At this optimum, then we have $x_{k+1}=x_k+V_{q_X,p}^l(x_k)=x_k$. This motivates the following training loss:
\begin{align}
\mathcal{L}^{Sinkhorn}:=\mathcal{L}^\infty:=\mathbb{E}_{\epsilon}[\|f_\theta(\epsilon)-\text{sg}(f_\theta(\epsilon)+V^l_{q_\theta,p}(f(\epsilon)))\|^2]\label{eq:sinkhorn_drift_loss}.
\end{align}
In practice, both $p=p_{\mathrm{data}}$ and $q_\theta:=(f_\theta)_\#p_\epsilon$ are intractable, so we work with empirical measures. Specifically, we draw $n$ i.i.d.\ samples $\epsilon^1,\ldots,\epsilon^n\sim p_\epsilon$ and $x^1,\ldots,x^n\sim p$, and define
$\hat{p}_\epsilon=\sum_{i=1}^n\delta_{\epsilon^i}$ and $\hat{p}=\sum_{i=1}^n\delta_{x^i}$. This yields $\hat{q}_\theta=\sum_{i=1}^n\delta_{f_\theta(\epsilon^i)}$, and the loss \eqref{eq:sinkhorn_drift_loss} admits the Monte Carlo approximation
\begin{align}
\mathcal{L}^{Sinkhorn}\approx\mathbb{E}_{\hat{p}_\epsilon,\hat{p}}\mathbb{E}_{\epsilon\sim \hat{p}_\epsilon}[\|f_\theta(\epsilon)-\text{sg}(f_\theta(\epsilon)+V^l_{\hat{q}_\theta,\hat{p}}(f(\epsilon)))\|^2]\label{eq:sinkhorn_drift_loss_mc}.
\end{align}
Here, the randomness comes from the i.i.d.\ samples $\boldsymbol{\epsilon}:=\{\epsilon^1,\ldots,\epsilon^n\}$ and the empirical distribution $\hat{p}$ obtained by i.i.d.\ sampling from $p_{\mathrm{data}}$.

Algorithm~\ref{alg:split_sinkhorn_drift} summarizes the computation of the Sinkhorn drifting field $V^\infty_{\hat{p},\hat{q}_\theta}$ used in our implementation. Here,
$\operatorname{RowNormalize}(\Pi)[i,j]
=\Pi[i,j]/\sum_k \Pi[i,k]$,
so that each row of $P$ sums to $1$ and defines a barycentric weight vector. Importantly, setting $T=1$ recovers Deng et al. \cite{deng2026generative}. Note that, in practice $Y_{neg}$ can be set to $X$, where for drift self-distances are masked, but Sinkhorn does not require self-distance masking.

\subsection{Discussion of the identity} 
\label{sec:identifiability} 
In this section, we discuss the case $V^\infty_{q,p}=0$. We start from the general setting $V^\infty_{q,p}\equiv 0$ on a compact set $\Omega$ and next consider the empirical approximation. 

\paragraph{Zero Sinkhorn drift in a smooth-density setting.}
We first record a continuous analogue in a regular regime where the feasible class is convex and the vanishing of the Wasserstein gradient implies full first-order stationarity.

\begin{proposition}\label{pro:identity_general_smooth}
Let $\Omega\subset\mathbb R^d$ be a \textbf{connected} compact domain, and let
$p=\rho_p(x)\,dx$ and $q=\rho_q(x)\,dx$ be probability measures on $\Omega$
with densities $\rho_p,\rho_q\in C^1(\Omega)$ and $\rho_q>0$ on $\Omega$.
Assume that $q\mapsto S_\tau(p,q)$ is differentiable at $q$. If $V^\infty_{q,p}\equiv 0$ on $\Omega$, then $q=p$.
\end{proposition}

Informally, the argument is as follows. In the continuous setting with smooth, strictly positive densities on a connected compact domain, the Sinkhorn drift is precisely the Wasserstein gradient-flow velocity associated with the functional $q \mapsto S_\tau(p,q)$, namely $V^\infty_{q,p}(x) = -\nabla_x \frac{\delta}{\delta q} S_\tau(p,q)(x)$. Therefore, if $V^\infty_{q,p}\equiv 0$, then the first-variation potential $\frac{\delta}{\delta q} S_\tau(p,q)$ must be spatially constant on $\Omega$ since $\Omega$ is connected. Since $q$ has a smooth, strictly positive density, every sufficiently small smooth zero-mass perturbation remains admissible; thus, the directional derivative of $S_\tau(p,\cdot)$ at $q$ vanishes in every feasible direction; in other words, $q$ is a stationary point of $q \mapsto S_\tau(p,q)$. Finally, the Sinkhorn divergence is definite and, for fixed $p$, strictly convex in $q$, so its unique stationary point and unique minimizer is $q=p$ \cite{feydy2019interpolating}. The appendix makes each of these steps precise.

Regarding the empirical approximation $V^\infty_{\hat{p},\hat{q}}$, we demonstrate the following:

\begin{proposition}
Under some regular conditions, let $\hat{p}$ and $\hat{q}$ be empirical distributions of size $n$, i.i.d.\ sampled from $p$ and $q$, respectively. Then
\begin{align}
\mathbb{E}\left[\|V_{q,p}^\infty - V^\infty_{\hat{q},\hat{p}}\|^2\right]
\lesssim \tau^{1-d/2}\log(n)\,n^{-1/2}.\nonumber
\end{align}
\end{proposition}
We refer to Appendix \ref{sec:sample_complexity} for the formal statement and proofs.

In practice, the above statistical conclusion implies that when the distribution is  supported in the connected compact domain, and the batch size $n$ is sufficiently large and $V_{\hat{p},\hat{q}}$ is sufficiently small, we have high confidence that $V^\infty_{q,p}=0$, which implies the model converges to the target distribution.

\paragraph{Zero Empirical Sinkhorn Drift implies the Identity.}
We introduce the following statement.

The above identity statement relies on sampling from the continuous true distribution $p,q$ rather than $\hat{p},\hat{q}$. Now we focus on the empirical-measure setting and discuss when the condition $V^{\infty}_{q,p}=0$ forces $p=q$. Throughout, we work under the following assumptions.
\begin{remark}\label{cond:identity}
Let $p = \frac{1}{n} \sum_{j=1}^n \delta_{y_j}$ be a fixed empirical measure supported on $n$ distinct points in $\mathbb{R}^d$. Let $q = \frac{1}{n} \sum_{i=1}^n \delta_{x_i}$ be an empirical measure with particle locations $x_i \in \mathbb{R}^d$.
\begin{itemize}
    \item Non-degeneracy: the points $\{x_i\}_{i=1}^n$ are pairwise distinct ($x_i \neq x_k$ for $i \neq k$).
    \item Stationarity: the particle gradient vanishes, $\nabla_{x_i} F(X) = 0$ for all $i = 1, \dots, n$.
\end{itemize}
\end{remark}

We start with the unregularized case $\tau=0$.

\begin{proposition}\label{pro:identity_tau=0}\label{thm:identity}
Define the Sinkhorn divergence objective $F(X)=S_\tau(p,q)$ and assume $\tau=0$. Under Remark~\ref{cond:identity}, we have $p=q$ (equivalently, $\{x_i\}$ is a permutation of $\{y_j\}$).
\end{proposition}

For the regularized regime $0<\tau<\infty$, we can prove identifiability when $n=2$.
\begin{proposition}
\label{pro:identity_n=2}
Suppose $n = 2$, with $x_1 \neq x_2$ and $y_1 \neq y_2$.
If $V^{\infty}_{q,p} = 0$, then $\{y_1, y_2\} = \{x_1, x_2\}$ as sets, hence $p_X = p_Y$.
\end{proposition}

\begin{remark}
In the same setting, the condition $V^{\mathrm{drift}}_{q,p}(x_i)=0$ for all $i$ does not imply $p=q$; see Appendix~\ref{sec:drift_model} for a counterexample and discussion. The authors of \cite{deng2026generative} also note (Section~6 and Appendix~C.1) that the converse implication $V_{q,p}\equiv 0 \Rightarrow p=q$ is \textbf{not guaranteed in general} for their construction.
Their identifiability argument (Appendix~C.1) further relies on a \emph{non-degeneracy assumption}: the bilinear interaction vectors $\{U_{ij}\}_{i<j}$ arising from a basis expansion of $p$ and $q$ are assumed to be linearly independent. This generic condition can fail for specific choices of kernel, test points, and particle configurations.
\end{remark}

For general $n\ge 3$, whether $V^{\infty}_{q,p}=0$ implies $p=q$ remains open. For $n\ge 3$, we show that $V^\infty(X)=0$ on $\text{supp}(q_X)$ implies that $q_X$ is a stationary point of the functional $q_X \mapsto S_\tau(p,q_X)$ when restricted to the submanifold of empirical measures supported on $n$ points; all proofs are deferred to Appendix~\ref{sec:identity_2}.






\section{Computational tradeoff}

Sinkhorn scaling makes each \emph{training} step more expensive than Drift-style one- (or partial two-sided) normalization, since it requires a small number of iterative updates rather than a single softmax pass. Importantly, this additional cost is confined to training: at \emph{inference} time, our method has the \emph{same} computational cost as drifting approaches, because generation still amounts to the same learned network evaluation.

By accepting slightly slower training, Sinkhorn offers two key benefits: improved stability with respect to the temperature parameter and stronger theoretical rigor (e.g., principled two-sided balancing and clearer identifiability/structure). In turn, this reduces reliance on ad hoc engineering tricks often needed to make drifting methods work in practice, such as masking self-distances, averaging the drift across a set of temperatures, or related heuristics.

\section{Numerical Studies}

\subsection{Drift Behavior for Varying $\tau$}
\label{sec:drift_tau}
We perform an $\tau$-sweep over $\{0.01, 0.1, 1.0, 10.0\}$ to study how the kernel temperature affects the drifting trajectories under three normalization schemes: (i)~one-sided normalization (softmax over $y$), (ii)~two-sided normalization ($\sqrt{A_{\mathrm{row}} \odot A_{\mathrm{col}}}$), and (iii)~full
Sinkhorn (100 iterations), using $N{=}100$ source and target samples with step size $\eta{=}0.1$ and $500$ Euler steps. The results for $\tau\in\{0.1,1.0\}$ are shown in Figure \ref{fig:drift_trajectories_eps_sweep}, while the full results are moved to Appendix~\ref{app:drift_tau} due to space constraints. At small $\tau$ (e.g., $= 0.01$), the kernel $k(x,y) = \exp(-\|x - y\|/\tau)$ is sharply peaked, so each source particle couples almost exclusively to its nearest target (for positive and negative) causing the drift to collapse, while the Sinkhorn drift in fact gets closer to the 2-Wasserstein gradient flow (i.e., flow in straight lines).  

To better see why, consider the negative (repulsion) term, $V^{-}_q$. Since the
negative samples $y^{-}$ are drawn from the same batch as $x$
(i.e., $y^{-} \in \{x_1, \ldots, x_N\}$), the self-distance
$\|x_i - x_i\| = 0$ yields a kernel value $k(x_i, x_i) =
\exp(0) = 1$, while for any $j \neq i$ with $\|x_i - x_j\| >
0$, we have $k(x_i, x_j) = \exp(-\|x_i - x_j\|/\tau)
\to 0$ as $\tau \to 0$. After normalization (e.g.,
softmax), virtually all weight concentrates on the
self-interaction, but this term contributes $(x_i - x_i) = 0$
to the mean-shift vector. Hence $V^{-}_q(x_i) \approx
\mathbf{0}$, and the repulsion effectively vanishes. Without
repulsion, the drift reduces to attraction-only, which
as shown by \cite{deng2026generative} could be
catastrophic. To avoid this degeneracy,
\cite{deng2026generative} masks out the self-interaction by
assigning a large surrogate self-distance (e.g., $10^6$) in the
negative logits (cf.\ \texttt{dist\_neg += eye(N) * 1e6} in
their Algorithm~2), ensuring that the softmax weight on the
diagonal is negligible and the repulsion term remains
well-defined even at low temperatures. While practical, this self masking completely changes the drift dynamics as shown in Figure \ref{fig:drift_trajectories_eps_sweep} (columns denoted with `mask').

\subsection{Toy Experiments}
\label{toy_exp}
\paragraph{Setup.}
To evaluate the three normalization schemes in a generative learning setting,
we train a two-layer MLP generator $f_\theta : \mathbb{R}^2 \to \mathbb{R}^2$
mapping Gaussian noise to 2D target distributions using the respective drifting
losses \eqref{eq:drift_loss} and~\eqref{eq:sinkhorn_drift_loss}.
We compare one-sided, two-sided, and Sinkhorn normalization across
$\tau \in \{0.01, 0.05, 0.1\}$ on two target distributions:
\textbf{8-Gaussians} (a mixture of 8 isotropic Gaussians arranged symmetrically,
posing a multimodal coverage challenge) and \textbf{Checkerboard} (a multi-connected planar distribution).
Each run uses $N=500$ samples per mini-batch and trains for $5{,}000$ iterations
with Adam (lr\,$=10^{-3}$).
We report the squared 2-Wasserstein distance $W_2^2$ between generated and target samples, evaluated every 100 steps.
\paragraph{Results.}
Figure~\ref{fig:toy_generative} shows the generated distribution at the final
iteration and the $W_2^2$ convergence curves.
\textit{Mode coverage:}
At $\tau=0.1$, one-sided normalization collapses to a single mode on 8-Gaussians
($W_2^2\approx 7$--$8$), while Sinkhorn covers all 8 modes and achieves
$W_2^2<1$.
Two-sided normalization partially improves over one-sided but remains unstable.
\textit{Temperature sensitivity:}
As $\tau$ decreases, one-sided and two-sided methods become increasingly prone
to mode collapse on 8-Gaussians, whereas Sinkhorn remains comparatively stable
across all three values of $\tau$.
This is consistent with the identifiability guarantee of
Theorem~\ref{thm:identity}: at small $\tau$ the Gibbs kernel concentrates mass
on nearest neighbors, and without global marginal balance the coupling
degenerates; Sinkhorn's two-sided constraints prevent this collapse.
On Checkerboard, all three methods achieve similar $W_2^2$, indicating
that the advantage of Sinkhorn is most pronounced for multimodal targets with
isolated modes.
Taken together, these results confirm that the theoretical precision of Sinkhorn
normalization translates to measurable empirical gains in the generative training
setting.
Additional results on a broader set of distributions and with the Laplacian
kernel are provided in Appendix~\ref{app:toy}.

\begin{figure}[t]
\centering
\includegraphics[width=\linewidth]{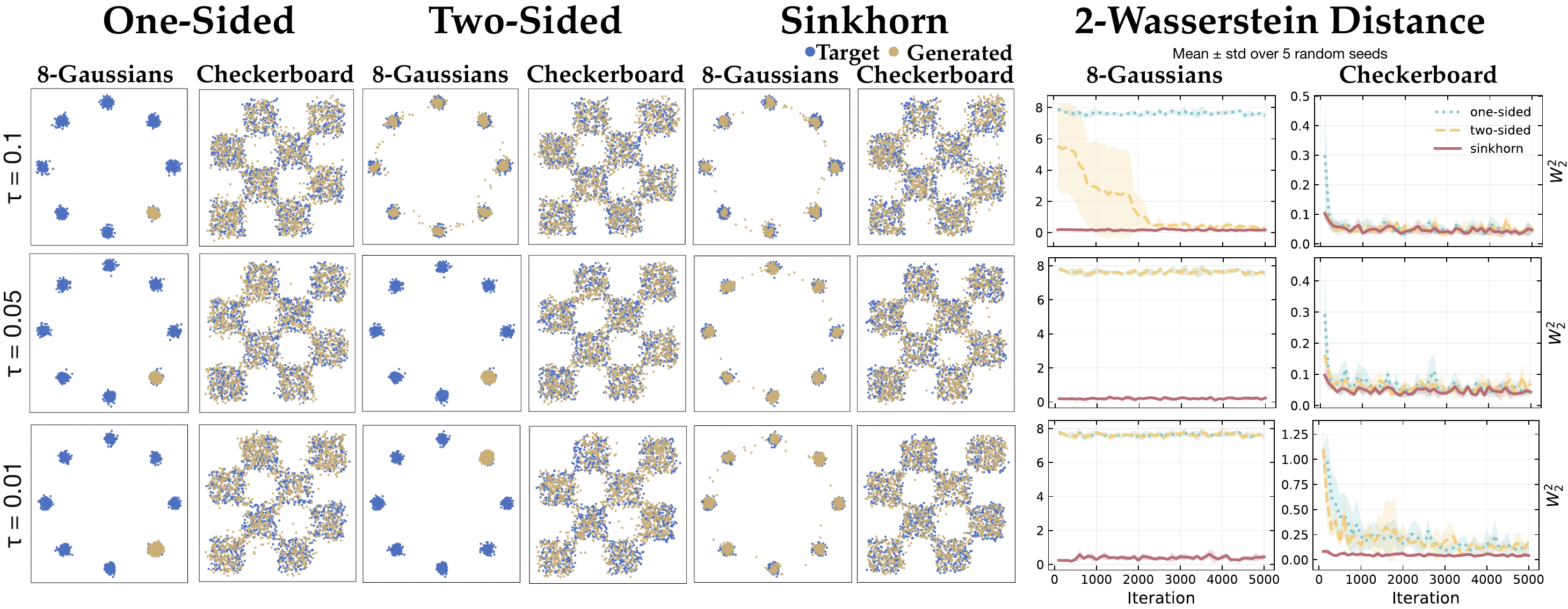}
\caption{Generative model training on 2D distributions across $\tau\in\{0.01,0.05,0.1\}$
and three normalization schemes (one-sided, two-sided, Sinkhorn) with Gaussian kernel.
\textbf{Left six columns}: final generated samples (orange) vs.\ target (blue).
\textbf{Right two columns}: $W_2^2$ convergence curves over $5{,}000$ iterations.
Sinkhorn consistently achieves lower $W_2^2$ and better mode coverage,
especially at small $\tau$.}
\label{fig:toy_generative}
\end{figure}

\subsection{MNIST Experiments}
\label{sec:mnist}
\paragraph{Setup.}
We evaluate class-conditional generation quality on MNIST across a temperature sweep
$\tau \in \{0.005, 0.01, 0.02, 0.025, 0.03, 0.04, 0.05, 0.1\}$,
using the Gaussian kernel $k(x,y)=\exp(-\|x-y\|^2/\tau)$.
We train standard drifting (Eq.~\eqref{eq:drift_loss}) with the geometric-mean
two-sided normalization used in the Drift implementation, and Sinkhorn-drifting
(Eq.~\eqref{eq:sinkhorn_drift_loss}) on MNIST in a 6-dimensional latent space
obtained from a convolutional autoencoder. Each generator is a 3-layer MLP
trained for $5{,}000$ steps with Adam.
We report the average per-class EMD (squared 2-Wasserstein distance $W_2^2$) in latent space and class-conditional generation accuracy.

\paragraph{Results.}
Table~\ref{tab:mnist_gaussian} and Figure~\ref{fig:mnist_samples} summarize
the results.
Sinkhorn is stable across the full temperature range: EMD
stays within $6.88$--$8.57$ and class accuracy remains $\geq\!99.97\%$ for all
$\tau \in [0.005, 0.1]$.
\textit{Baseline collapses at small $\tau$}: for every $\tau \leq 0.05$, the baseline generator degenerates to a single mode (EMD $\approx 73.2$--$79.7$) and class accuracy drops to random chance ($\approx\!10\%$), indicating complete loss of class conditioning. Only at $\tau=0.1$ does the baseline recover (EMD $=5.63$, Acc $=95.0\%$), while Sinkhorn already achieves stable performance from $\tau=0.005$. This collapse is consistent with the identifiability gap analyzed in
Section~\ref{sec:identifiability}: at small $\tau$, the one-sided row-normalized
coupling concentrates all weight on the nearest neighbor, collapsing the
repulsive term and causing mode collapse.
Sinkhorn's doubly-stochastic marginal constraints preclude this degenerate
solution.
Additional results with the Laplacian kernel are in
Appendix~\ref{app:mnist_laplacian}.

\begin{table}[t]
\centering
\caption{MNIST $\tau$-sweep with Gaussian kernel.
$^\dagger$Acc\,$\approx\!10\%$ indicates mode collapse.}
\label{tab:mnist_gaussian}
\scriptsize
\setlength{\tabcolsep}{1pt}
\renewcommand{\arraystretch}{1.03}
\begin{tabular*}{\linewidth}{@{\extracolsep{\fill}}c c c c c @ {\hspace{1pt}\vrule width 0.2pt\hspace{1pt}} c c c c c@{}}
\toprule
& \multicolumn{2}{c}{EMD $\downarrow$} & \multicolumn{2}{c}{Accuracy $\uparrow$}
& & \multicolumn{2}{c}{EMD $\downarrow$} & \multicolumn{2}{c}{Accuracy $\uparrow$} \\
\cmidrule(lr){2-3}\cmidrule(lr){4-5}\cmidrule(lr){7-8}\cmidrule(lr){9-10}
$\tau$ & {Baseline} & {Sinkhorn} & {Baseline} & {Sinkhorn}
& $\tau$ & {Baseline} & {Sinkhorn} & {Baseline} & {Sinkhorn} \\
\midrule
0.005 & 73.21 & \bfseries 8.57 & $9.97\%^\dagger$ & \bfseries 100.00\%
& 0.030 & 73.22 & \bfseries 6.96 & $9.95\%^\dagger$ & \bfseries 100.00\% \\
0.010 & 73.21 & \bfseries 7.71 & $9.99\%^\dagger$ & \bfseries 100.00\%
& 0.040 & 77.19 & \bfseries 6.91 & $10.00\%^\dagger$ & \bfseries 100.00\% \\
0.020 & 73.21 & \bfseries 7.12 & $9.98\%^\dagger$ & \bfseries 100.00\%
& 0.050 & 79.73 & \bfseries 6.88 & $9.96\%^\dagger$ & \bfseries 100.00\% \\
0.025 & 73.21 & \bfseries 7.01 & $9.99\%^\dagger$ & \bfseries 100.00\%
& 0.100 & \bfseries 5.63 & 6.91 & 95.00\% & \bfseries 100.00\% \\
\bottomrule
\end{tabular*}
\end{table}

\begin{figure*}[t]
\centering
\includegraphics[width=\textwidth]{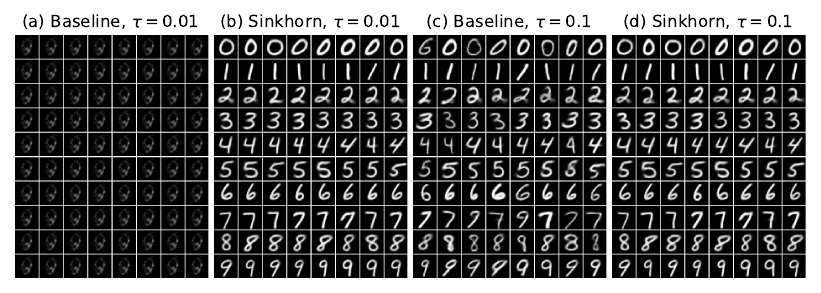}
\caption{Generated MNIST samples (Gaussian kernel).
Each panel shows 10 classes $\times$ 8 samples.
\textbf{(a)} Baseline at $\tau=0.01$ collapses to a single degenerate mode;
class accuracy is ${\approx}10\%$ (random chance).
\textbf{(b)} Sinkhorn at $\tau=0.01$ correctly generates all ten classes
with 100\% class accuracy.
\textbf{(c,d)} At $\tau=0.1$ both methods produce recognizable digits,
with Sinkhorn remaining sharper and more consistent.}
\label{fig:mnist_samples}
\end{figure*}

\begin{table*}[h!]
\centering
\scriptsize
\caption{
  Class-conditional face generation on FFHQ with ALAE across temperatures $\tau$.
  We report EMD$\downarrow$ and FID$\downarrow$ for Baseline and Sinkhorn. Average is the mean value of all 6 classes for each $\tau$.
}
\label{tab:alae}
\setlength{\tabcolsep}{2.8pt}
\renewcommand{\arraystretch}{1.08}
\resizebox{\textwidth}{!}{%
\begin{tabular}{c l cc cc @{\hspace{6pt}} l cc cc}
\toprule
\multirow{3}{*}{$\tau$} & \multicolumn{5}{c}{\textbf{Male}} & \multicolumn{5}{c}{\textbf{Female}} \\
\cmidrule(lr){2-6}\cmidrule(lr){7-11}
& \multirow{2}{*}{Class} & \multicolumn{2}{c}{EMD$\downarrow$} & \multicolumn{2}{c}{FID$\downarrow$}
& \multirow{2}{*}{Class} & \multicolumn{2}{c}{EMD$\downarrow$} & \multicolumn{2}{c}{FID$\downarrow$} \\
& & Baseline & Sinkhorn & Baseline & Sinkhorn
& & Baseline & Sinkhorn & Baseline & Sinkhorn \\
\midrule
\multirow{4}{*}{0.1}
& Adult    & 441.1 & \textbf{145.0} & 198.7 & \textbf{33.8}
& Adult    & 414.0 & \textbf{146.5} & 143.7 & \textbf{39.7} \\
& Children & 489.8 & \textbf{157.6} & 228.5 & \textbf{40.8}
& Children & 465.8 & \textbf{149.1} & 185.6 & \textbf{38.6} \\
& Old      & 508.7 & \textbf{138.5} & 219.1 & \textbf{31.4}
& Old      & 400.6 & \textbf{130.0} & 150.5 & \textbf{38.3} \\
\cmidrule(lr){2-11}
& \textbf{Average} & 453.3 & \textbf{144.4} & 187.7 & \textbf{37.1} & & & & & \\
\midrule
\multirow{4}{*}{1.0}
& Adult    & 325.8 & \textbf{142.2} & 136.1 & \textbf{33.5}
& Adult    & 362.1 & \textbf{142.1} & 152.3 & \textbf{35.7} \\
& Children & 426.5 & \textbf{148.7} & 206.4 & \textbf{35.2}
& Children & 387.8 & \textbf{142.6} & 180.4 & \textbf{34.7} \\
& Old      & 380.8 & \textbf{132.1} & 102.0 & \textbf{30.0}
& Old      & 340.2 & \textbf{118.1} & 100.4 & \textbf{33.1} \\
\cmidrule(lr){2-11}
& \textbf{Average} & 370.5 & \textbf{137.6} & 146.3 & \textbf{33.7} & & & & & \\
\midrule
\multirow{4}{*}{10.0}
& Adult    & 159.5 & \textbf{143.4} & 40.5 & \textbf{35.3}
& Adult    & 163.3 & \textbf{142.4} & 54.1 & \textbf{36.2} \\
& Children & 174.6 & \textbf{149.1} & 50.5 & \textbf{34.8}
& Children & 162.7 & \textbf{143.4} & 49.1 & \textbf{34.3} \\
& Old      & 148.9 & \textbf{133.6} & 38.9 & \textbf{30.3}
& Old      & 129.3 & \textbf{118.9} & 38.8 & \textbf{32.3} \\
\cmidrule(lr){2-11}
& \textbf{Average} & 156.4 & \textbf{138.5} & 45.3 & \textbf{33.9} & & & & & \\
\bottomrule
\end{tabular}%
}
\end{table*}

\subsection{Image Generation Experiments}
\label{sec:image_gen}

\paragraph{Setup.}
We evaluate class-conditional image generation on FFHQ~\cite{karras2019style} with a pretrained
Adversarial Latent Autoencoder (ALAE)~\cite{pidhorskyi2020adversarial}.
In our pipeline, we train and evaluate in the ALAE latent space ($\mathbb{R}^{512}$); image-space
metrics are computed after decoding with the same frozen ALAE decoder at $1024\times1024$.
We consider six demographic classes:
\{\textit{Male-Children}, \textit{Male-Adult}, \textit{Male-Old},
\textit{Female-Children}, \textit{Female-Adult}, \textit{Female-Old}\},
and train a conditional latent generator
$f_\theta:\mathbb{R}^{512+64}\rightarrow\mathbb{R}^{512}$
(3-layer MLP, hidden width $1024$) with learned class embeddings.
We compare
(1) the baseline drifting loss~\eqref{eq:drift_loss} with the geometric-mean two-sided normalization used in the Drift implementation,
and
(2) Sinkhorn drifting~\eqref{eq:sinkhorn_drift_loss},
across $\tau\in\{0.1,\,1.0,\,10.0\}$.
For each checkpoint, we sample $1{,}000$ generated latents per class and match them with $1{,}000$ real latents per class from the same class-specific real latent pool.
We report latent EMD and image FID~\cite{heusel2017gans}
computed between decoded generated image and decoded real images.
Full per-class results are reported in Table~\ref{tab:alae}.
Representative qualitative comparisons are shown in Figure~\ref{fig:alae_qualitative},
with additional panels in Appendix~\ref{sec:appendix_alae}.

\paragraph{Results.}
Table~\ref{tab:alae} shows that Sinkhorn drifting consistently outperforms the baseline across all temperatures and all six classes in both EMD (squared 2-Wasserstein distance $W_2^2$) and FID.
At $\tau{=}0.1$, the gap is largest: mean FID drops from $187.7$ to $37.1$ (about $5.1\times$), and mean EMD drops from $453.3$ to $144.4$ (about $3.1\times$).
At $\tau{=}1.0$, Sinkhorn remains clearly better, with mean FID $33.7$ vs.\ $146.3$ (about $4.3\times$) and mean EMD $137.6$ vs.\ $370.5$ (about $2.7\times$).
At $\tau{=}10.0$, both methods improve and the gap narrows, but Sinkhorn is still better (mean FID $33.9$ vs.\ $45.3$, mean EMD $138.5$ vs.\ $156.4$).
Qualitatively, Figure~\ref{fig:alae_qualitative} shows randomly picked samples for all-classes at $\tau{=}1.0$ and $\tau{=}10.0$; each row is one class, with baseline on the left and Sinkhorn on the right.

The $\tau{=}0.1$ qualitative panel (where the difference is most pronounced) is provided in Appendix~\ref{sec:appendix_alae}.

\section*{Conclusion}
We established a theoretical connection between drifting generative dynamics and the Wasserstein gradient flow of the Sinkhorn divergence. At the particle level, both share the same cross–minus–self barycentric structure, differing only in how the coupling matrices are constructed: drifting relies on one-sided kernel normalization, while Sinkhorn divergence uses doubly-stochastic couplings obtained through entropic optimal transport. This view shows that drifting can be interpreted as a single-iteration approximation of the Sinkhorn gradient flow and explains the empirical benefits of partial two-sided normalization. Importantly, the Sinkhorn formulation resolves the identifiability gap of drifting models by ensuring that vanishing drift implies equality of the model and target distributions. Experiments on synthetic distributions, MNIST, and FFHQ demonstrate improved stability and mode coverage, particularly at low temperatures, while preserving the one-step inference procedure of drifting models.

\paragraph{Limitations.} We acknowledge that our experiments are smaller in scale than those of Deng et al.~\cite{deng2026generative}, largely due to the heavy compute required by benchmarks such as ImageNet-1K generation. Nonetheless, our theoretical results and extensive smaller-scale experiments consistently demonstrate improved performance without the engineering heuristics often needed by drifting methods.

\section{Acknowledgments}
This work was partially supported by NSF CAREER Award No. 2339898. The authors also gratefully acknowledge the computational resources provided through the NVIDIA Academic Grant Program and Lambda cloud credits. Finally, YB thanks Dr. Guang Lin of Purdue University for his helpful discussions and support.
\begin{figure}[t!]
  \centering
  \includegraphics[width=\linewidth]{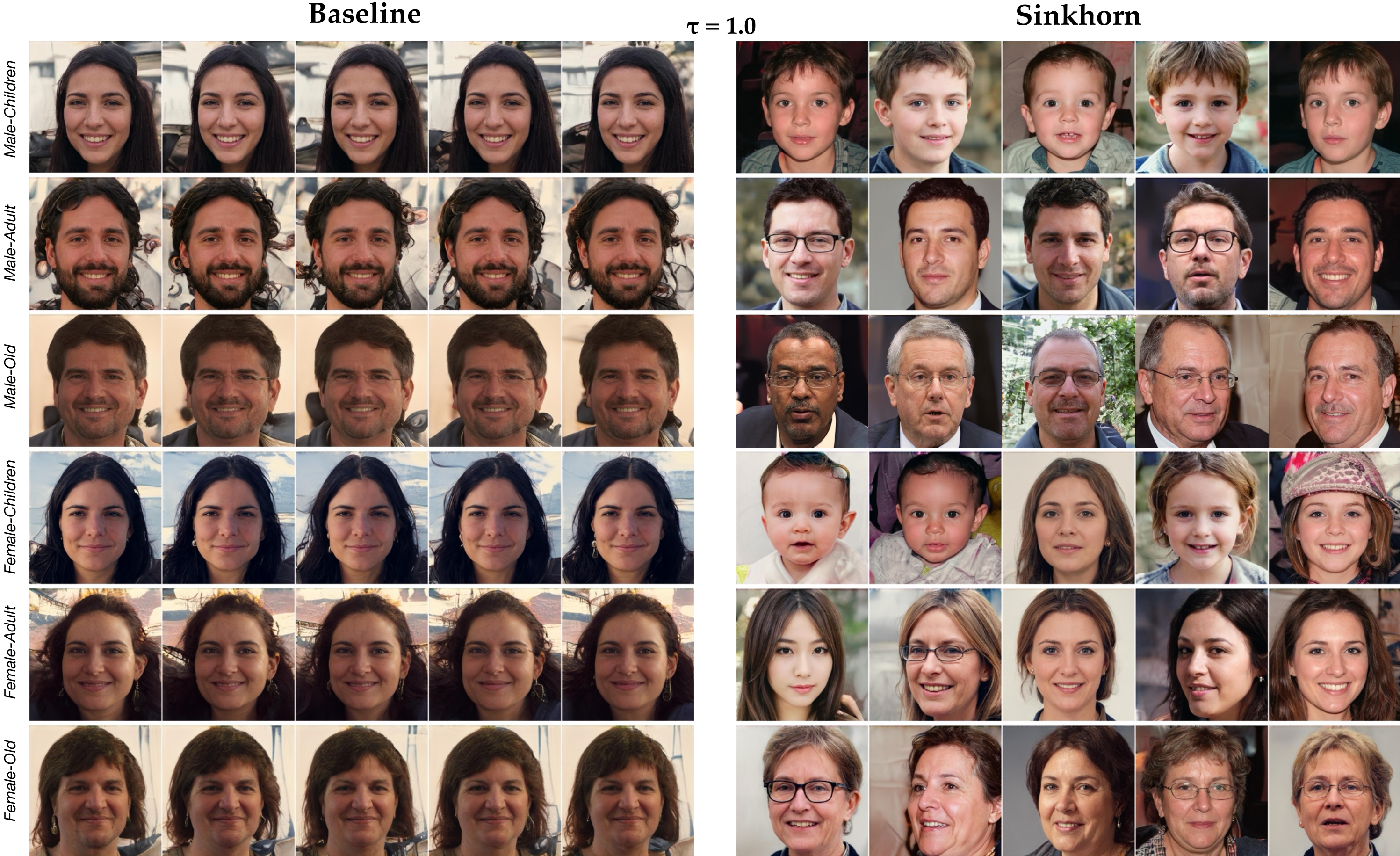}
  \vspace{3pt}
  \includegraphics[width=\linewidth]{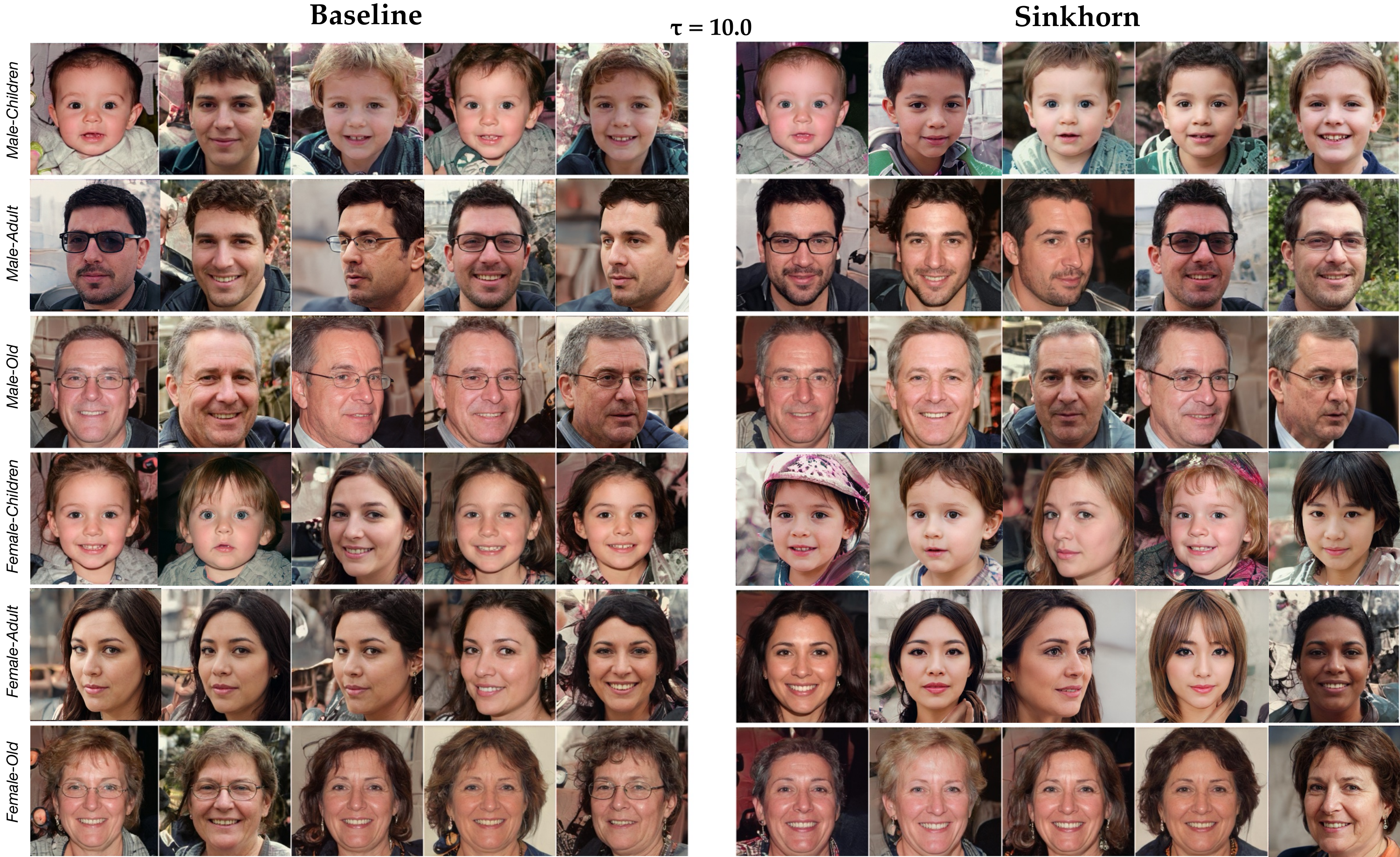}
  \caption{
    Qualitative comparison of class-conditional FFHQ generation at $\tau{=}1.0$ (top) and $\tau{=}10.0$ (bottom).
    In each panel, each row corresponds to one class; \textbf{Baseline is on the left} and \textbf{Sinkhorn is on the right}.
    The corresponding low-temperature qualitative panel ($\tau{=}0.1$) is shown in Figure~\ref{fig:alae_appendix_tau01} of Appendix~\ref{sec:appendix_alae}.
  }
  \label{fig:alae_qualitative}
\end{figure}
\FloatBarrier

\clearpage

\bibliographystyle{plainnat}
\bibliography{refs}

\clearpage
\appendix
\section{Understanding Drift Generative Model}\label{sec:drift_model}

\subsection{Empirical form and ``attention'' view of Drift Field.}
Let $p=\frac1n\sum_{i=1}^n \delta_{y_i}$ and $q=\frac1n\sum_{j=1}^n \delta_{x_j}$ be empirical measures with samples
$Y=\{y_i\}_{i=1}^n$ and $X=\{x_j\}_{j=1}^n$.
Define Gibbs affinities (e.g.\ with a cost $C$ and temperature $\tau>0$)
\begin{equation}
K_{XY}[j,i] := \exp\big(-C(x_j,y_i)/\tau\big),
\qquad
K_{XX}[j,i] := \exp\big(-C(x_j,x_i)/\tau\big).\nonumber
\end{equation}
Row-normalization yields row-stochastic matrices

\begin{equation}
P^{\mathrm{drift}}_{XY}[j,i]:= \frac{K_{XY}[j,i]}{\sum_{\ell=1}^n K_{XY}[j,\ell]},
\qquad
P^{\mathrm{drift}}_{XX}[j,i] :=\frac{K_{XX}[j,i]}{\sum_{\ell=1}^n K_{XX}[j,\ell]},
\label{eq:drift_row_norm}
\end{equation}
which can be viewed as distance-based attention weights (softmin over costs).
Then the Drift field evaluated on particles is
\begin{align}
V_{\mathrm{drift}}(x_j)
&=
\sum_{i=1}^n P^{\mathrm{drift}}_{XY}[j,i](y_i-x_j)
-
\sum_{k=1}^n P^{\mathrm{drift}}_{XX}[j,i](x_i-x_j)\nonumber \\ 
&=\sum_{i=1}^n P^{\mathrm{drift}}_{XY}[j,i]y_i-
\sum_{i=1}^n P^{\mathrm{drift}}_{XX}[j,i]x_i
\label{eq:drift_particles_vfield}
\end{align}
where the second equation follows from the fact $\sum_{i}P_{XY}^{\text{drift}}[j,i]=\sum_{i}P_{XX}^{\text{drift}}[j,i]=1,\forall j$. 
Equivalently, in matrix form (stack $X,Y\in\R^{N\times d}$ row-wise),
\begin{equation}
V_{\mathrm{drift}}(X)
=
P^{\mathrm{drift}}_{XY}Y - P^{\mathrm{drift}}_{XX}X.
\label{eq:drift_matrix_form}
\end{equation}
Thus Drift implements a \emph{cross minus self} barycentric projection, but with \emph{one-sided (row) normalization}.

\subsection{Identifiability of Drift Generative Model}
In this section, we provide a counterexample showing that for \cite{deng2026generative}, the condition
$V_{q,p}(x)=0$ on the support of $p$ does not imply $p=q$.

Let $\tau=1$ and $k(x,y)=\exp(-(x-y)^2)$. Consider the empirical measures
$$
p=\tfrac12\delta_0+\tfrac12\delta_1,\qquad
q=\tfrac12\delta_a+\tfrac12\delta_b.
$$
Recall $V_{q,p}(x)$ is defined in \eqref{eq:drift_particles_vfield}. We focus on the two equations
$V_{q,p}(0)=0$ and $V_{q,p}(1)=0$.

\begin{lemma}[Non-identifiability on the support of $p$]
There exists $(a^\ast,b^\ast)\in\mathcal{R}:=[-1.5,-1.2]\times[0.6,0.9]$ such that
$$
V_{q,p}(0)=0,\qquad V_{q,p}(1)=0,
$$
and hence $q\neq p$.
\end{lemma}

\begin{proof}
Define $F_1(a,b)=V_{q,p}(0)$ and $F_2(a,b)=V_{q,p}(1)$.
By expanding the expectation in the definition of $V_{q,p}$ (a finite sum since $p,q$ are empirical),
one obtains the following explicit expressions (up to a positive multiplicative constant):
\begin{align}
F_1(a,b)
&=
(-a)e^{-a^2}+(-b)e^{-b^2}+(1-a)e^{-a^2-1}+(1-b)e^{-b^2-1},\label{eq:F1_def}\\
F_2(a,b)
&=
(-a)e^{-1-(1-a)^2}+(-b)e^{-1-(1-b)^2}+(1-a)e^{-(1-a)^2}+(1-b)e^{-(1-b)^2}.\label{eq:F2_def}
\end{align}
In particular, $F_1$ and $F_2$ are continuous on $\mathbb{R}^2$.

Let $\mathcal R=[-1.5,-1.2]\times[0.6,0.9]$.
We claim that the following sign conditions hold on the boundary of $\mathcal R$:
\begin{align}
&F_1(-1.5,b)<0,\qquad F_1(-1.2,b)>0,\qquad \forall b\in[0.6,0.9],\label{eq:sign_F1}\\
&F_2(a,0.6)>0,\qquad F_2(a,0.9)<0,\qquad \forall a\in[-1.5,-1.2].\label{eq:sign_F2}
\end{align}
The inequalities \eqref{eq:sign_F1}--\eqref{eq:sign_F2} can be verified \emph{rigorously} using interval arithmetic:
for each boundary segment, we compute an interval enclosure of $F_1$ or $F_2$ over the entire segment and confirm
that the resulting interval is strictly negative or strictly positive, implying that the sign does not change on that segment.

Given \eqref{eq:sign_F1}--\eqref{eq:sign_F2} and continuity of $(F_1,F_2)$, Miranda's theorem (a two-dimensional intermediate value theorem)
guarantees the existence of $(a^\ast,b^\ast)\in\mathcal R$ such that
$$
F_1(a^\ast,b^\ast)=0,\qquad F_2(a^\ast,b^\ast)=0,
$$
i.e., $V_{q,p}(0)=V_{q,p}(1)=0$.

Finally, since $\mathcal R\cap\{0,1\}^2=\varnothing$, we have $(a^\ast,b^\ast)\notin\{0,1\}^2$, hence $q\neq p$.
\end{proof}

\subsection{Gradient Descent in Drift Field}
We briefly formalize the connection between the Drift loss and the drift field $V_{\mathrm{drift}}$.

\begin{proposition}[Stop-gradient drift loss induces the drift ODE in particle space]
\label{pro:drift_particle_ode}
Let $x^1,\ldots,x^n\in\mathbb{R}^d$ be free particles (no parametrization restriction)
and let $V^{\mathrm{drift}}:\mathbb{R}^d\to\mathbb{R}^d$ be the Drift field.
Consider the stop-gradient objective
\begin{equation}
\mathcal{L}_{\mathrm{drift}}(x^1,\ldots,x^n)
:=\frac12\sum_{i=1}^n\Big\|x^i-\mathrm{sg}\big(x^i+V^{\mathrm{drift}}(x^i)\big)\Big\|^2.\nonumber 
\end{equation}
Then the gradient with respect to each particle is
\begin{equation}
\nabla_{x^i}\mathcal{L}_{\mathrm{drift}}
= -\,V^{\mathrm{drift}}(x^i),
\label{eq:drift_particle_grad}
\end{equation}
and hence the (continuous-time) gradient flow in particle space satisfies
\begin{equation}
\dot x^i_t = -\nabla_{x^i}\mathcal{L}_{\mathrm{drift}}(x_t)
= V^{\mathrm{drift}}(x^i_t).
\label{eq:drift_particle_ode}
\end{equation}
\end{proposition}

\begin{proof}
Fix $i$ and denote $t_i:=\mathrm{sg}(x^i+V^{\mathrm{drift}}(x^i))$.
By definition of $\mathrm{sg}(\cdot)$, $t_i$ is treated as constant when differentiating
with respect to $x^i$. Therefore,
$$
\nabla_{x^i}\,\frac12\|x^i-t_i\|^2 = x^i-t_i.
$$
Since $x^i-t_i = -V^{\mathrm{drift}}(x^i)$, we obtain \eqref{eq:drift_particle_grad},
and \eqref{eq:drift_particle_ode} follows from the definition of gradient flow.
\end{proof}

\begin{proposition}[Gradient of the Drift stop-gradient loss for $f_\theta$,]
\label{pro:drift_param_grad}
Let $x^i(\theta)=f_\theta(\epsilon^i)$ be generator outputs and let
$V^{\mathrm{drift}}(x)$ be the Drift field.
Consider the stop-gradient regression objective
\begin{equation}
\mathcal{L}_{\mathrm{drift}}(\theta)
:=\frac12\sum_{i=1}^n\Big\|f_\theta(\epsilon^i)-\mathrm{sg}\big(f_\theta(\epsilon^i)+V^{\mathrm{drift}}(f_\theta(\epsilon^i))\big)\Big\|^2.\nonumber 
\end{equation}
Then its gradient is
\begin{equation}
\nabla_\theta\mathcal{L}_{\mathrm{drift}}(\theta)
= -\sum_{i=1}^n J f_\theta(\epsilon^i)^\top\,V^{\mathrm{drift}}(x^i(\theta)),
\label{eq:drift_sg_grad_param}
\end{equation}
where $J f_\theta(\epsilon^i)\in\R^{d\times \dim(\theta)}$ denotes the Jacobian.
\end{proposition}

\begin{proof}
Let $x^i(\theta):=f_\theta(\epsilon^i)$ and
$t_i:=\mathrm{sg}(x^i(\theta)+V^{\mathrm{drift}}(x^i(\theta)))$.
By definition of $\mathrm{sg}(\cdot)$, $t_i$ is treated as constant when differentiating
with respect to $\theta$. Hence
$$
\nabla_\theta \frac12\|x^i(\theta)-t_i\|^2
= J x^i(\theta)^\top(x^i(\theta)-t_i).
$$
Since $x^i(\theta)-t_i=-V^{\mathrm{drift}}(x^i(\theta))$, summing over $i$ yields
\eqref{eq:drift_sg_grad_param}.
\end{proof}

\begin{remark}[Parametrization-induced deviation from the drift ODE]
Proposition~\ref{pro:drift_particle_ode} shows that, if the particles $x^i$ were optimized
directly, the stop-gradient loss induces the drift ODE $\dot x^i = V^{\mathrm{drift}}(x^i)$.
Under the parametrization $x^i=f_\theta(\epsilon^i)$, the induced output-space velocity
depends on the Jacobian through $\dot x^i = J f_\theta(\epsilon^i)\,\dot\theta$.
Therefore, updating $\theta$ does not in general guarantee the explicit Euler step
$x^i \mapsto x^i + \eta V^{\mathrm{drift}}(x^i)$ unless the parametrization and the update rule
can realize $V^{\mathrm{drift}}$ in the output space.
\end{remark}

\section{Background: Sinkhorn Algorithm in the General Measure Setting}
\label{sec:entropic_ot_sinkhorn_general}

Let $\alpha,\beta \in \mathcal{P}(\mathbb{R}^d)$ be Borel probability measures
and consider the entropic optimal transport problem
\begin{equation}
\min_{\pi \in \Pi(\alpha,\beta)}
\int_{\mathbb{R}^d \times \mathbb{R}^d}
c(x,y)\, d\pi(x,y)
+
\tau \mathrm{KL}(\pi \| \alpha \otimes \beta),\label{eq:entropic_ot_general} 
\end{equation}

\paragraph{Absolute continuity.}
Since the KL term is finite only if $\pi \ll \alpha\otimes\beta$,
any minimizer $\pi^\star$ of \eqref{eq:entropic_ot_general}
necessarily satisfies
$$
\pi^\star \ll \alpha\otimes\beta.
$$
Hence there exists a nonnegative measurable function
$f_\pi \in L^1(\alpha\otimes\beta)$ such that
\begin{equation}
\label{eq:rn_density}
d\pi(x,y)
=
f_\pi(x,y)
\, d\alpha(x)\, d\beta(y).
\end{equation}

\subsection{Sinkhorn iterations in the general setting}

The entropic OT problem \eqref{eq:entropic_ot_general}
is therefore equivalent to minimizing over
nonnegative functions $f \in L^1(\alpha\otimes\beta)$
\begin{equation}
\min_{f \ge 0}
\int c(x,y) f(x,y)\, d\alpha(x)d\beta(y)
+
\tau
\int f(x,y)\log f(x,y)\, d\alpha(x)d\beta(y),\nonumber
\end{equation}
subject to the marginal constraints 
\begin{align}
\int f_\pi(x,y)\, d\beta(y)
&= 1
\quad \text{for } \alpha\text{-a.e. } x,
\label{eq:marginal_x}
\\
\int f_\pi(x,y)\, d\alpha(x)
&= 1
\quad \text{for } \beta\text{-a.e. } y.
\label{eq:marginal_y}
\end{align}
which is equivalent to $f \alpha\otimes \beta \in \Pi(\alpha,\beta)$. 

\paragraph{Initialization.}
Analogous to the discrete case, we start from the Gibbs density
\begin{equation}
f^{(0)}(x,y) := e^{-\frac{C(x,y)}{\tau}}\nonumber 
\end{equation}

\paragraph{Alternating marginal normalizations.}
The Sinkhorn algorithm alternately enforces the two marginal
constraints by normalizing along one variable at a time.

For $\ell \ge 1$, define
\begin{equation}
\begin{cases}
\displaystyle
f^{(\ell)}(x,y)
=
\frac{
f^{(\ell-1)}(x,y)
}{
\int f^{(\ell-1)}(x,y')\, d\beta(y')
},
& \text{if $\ell$ is odd},\\
\displaystyle
f^{(\ell)}(x,y)
=
\frac{
f^{(\ell-1)}(x,y)
}{
\int f^{(\ell-1)}(x',y)\, d\alpha(x')
},
& \text{if $\ell$ is even}.
\end{cases}
\label{eq:sinkhorn_general}
\end{equation}

Each odd iteration enforces the constraint
\eqref{eq:marginal_x}, while each even iteration enforces
\eqref{eq:marginal_y}.

\paragraph{Associated transport plan.}
At iteration $\ell$, the corresponding coupling is
\begin{equation}
d\pi^{(\ell)}(x,y)
=
f^{(\ell)}(x,y)
\, d\alpha(x)\, d\beta(y).\label{eq:pi^l}
\end{equation}

Under mild integrability conditions on $c$,
the sequence $(f^{(\ell)})_{\ell\ge0}$
converges in $L^1(\alpha\otimes\beta)$
to the unique minimizer of
\eqref{eq:entropic_ot_general}.

\section{Background: Probability paths and gradient flows in \texorpdfstring{$\mathbb{R}^d$}{R\textasciicircum d}}
\label{sec:prob_paths_gf}

When the path is induced by transporting mass along a time-dependent velocity field $v_t:\mathbb{R}^d\to\mathbb{R}^d$ and $q_t$ admits a density (still denoted $q_t$), it satisfies the continuity equation
\begin{equation}
\partial_t q_t + \nabla\cdot(q_t v_t)=0,
\label{eq:continuity}
\end{equation}
understood in the weak sense.

\paragraph{First variation.}
Let $\mathcal{F}$ be an energy functional on probability measures. Its \emph{first variation} (a.k.a. functional derivative) $\frac{\delta \mathcal{F}}{\delta q}(q)$ is defined (when it exists) by the property that, for any signed perturbation $r$ with $\int r=0$,
\begin{equation}
\left.\frac{d}{d\varepsilon}\,\mathcal{F}(q+\varepsilon r)\right|_{\varepsilon=0}
\;=\; \int \frac{\delta \mathcal{F}}{\delta q}(q)(x)\,dr(x).
\label{eq:first_variation_def}
\end{equation}
Intuitively, $\frac{\delta \mathcal{F}}{\delta q}(q)$ plays the role of the gradient of $\mathcal{F}$ with respect to the density $q$.

\paragraph{Wasserstein gradient flow.}
The (formal) $2$-Wasserstein gradient flow of $\mathcal{F}$ is the steepest-descent evolution in the $W_2$ geometry and can be written as
\begin{equation}
\partial_t q_t
= -\mathrm{grad}_{W_2}\mathcal{F}(q_t)
:=\nabla\cdot\left(q_t \, \nabla \frac{\delta \mathcal{F}}{\delta q}(q_t)\right).
\label{eq:wgf_formal2}
\end{equation}
Comparing \eqref{eq:wgf_formal2} with \eqref{eq:continuity} shows that the corresponding velocity field is
\begin{equation}
v_t(x)= -\nabla \frac{\delta \mathcal{F}}{\delta q}(q_t)(x).
\label{eq:velocity_first_variation}
\end{equation}

\section{Details of Sinkhorn Drift Flow}\label{sec:sinkhorn_wgf}
\subsection{Proof of Proposition \ref{pro:sinkhorn_wgf}}
\begin{proof}
Define
$$
\Phi(P;X)
=
\langle P, C(X,Y)\rangle
+
\tau \sum_{i,j} P_{ij}(\log P_{ij}-1),
$$
where $C_{ij}(X,Y)=c(x^i,y^j)$ and
$$
\Pi
=
\{P\ge 0: P\mathbf 1=\tfrac1n\mathbf 1,\;P^\top\mathbf 1=\tfrac1n\mathbf 1\}.
$$

Then
$$
\mathrm{OT}_\tau(\hat p_{\mathrm{data}},q_X)
=
\min_{P\in\Pi}
\Phi(P;X).
$$

\medskip
\noindent
\textbf{Step 1: Envelope argument.}

Since $\Pi$ is compact and does not depend on $X$, 
$\Phi(P;X)$ is continuously differentiable in $X$, 
and the entropy term makes $\Phi$ strictly convex in $P$, 
the minimizer $\pi_{XY}^\infty(X)$ is unique.

Therefore, by Danskin's theorem,

$$
\nabla_X \mathrm{OT}_\tau(\hat p_{\mathrm{data}},q_X)
=
\nabla_X \Phi(\pi_{XY}^\infty(X);X).
$$

In particular, the derivative of $\pi_{XY}^\infty(X)$ does not appear.

\medskip
\noindent
\textbf{Step 2: Derivative with respect to particles.}

Only the cost term depends on $X$, hence

$$
\nabla_{x^i}
\mathrm{OT}_\tau(\hat p_{\mathrm{data}},q_X)
=
\sum_{j=1}^n
(\pi_{XY}^\infty)_{ij}
\nabla_{x} c(x^i,y^j).
$$

Similarly,
$$
\nabla_{x^i}\frac{1}{2}
\mathrm{OT}_\tau(q_X,q_X)
=
\sum_{j=1}^n
(\pi_{XX}^\infty)_{ij}\nabla_x c(x^i,x^j).
$$

For quadratic cost $c(x,y)=\frac{1}{2}\|x-y\|^2$, we have $\nabla_x c(x,y)=x-y$. 

Combining terms, we obtain: 

$$
\nabla_{x^i}\mathcal F(q_X)
=
\sum_{j=1}^n
(\pi_{XY}^\infty)_{ij}(x^i-y^j)
-
\sum_{j=1}^n
(\pi_{XX}^\infty)_{ij}(x^i-x^j)=-V_{\hat{p},\hat{q}_X}(x^i)
$$

The Partial ODE obtained from Wasserstein gradient flow \eqref{eq:wgf_particle}
$$
\dot x^i
=
-
\frac{1}{q_i}\nabla_{x^i}\mathcal F(q_X),
$$
becomes $\dot{x}^i=V_{\hat{p},\hat{q}_X}(x)$. 
\end{proof}
\subsection{The Sinkhorn flow in a general probability measure setting}

In this section, we extend Proposition~\ref{pro:sinkhorn_wgf} to the general probability measure setting.

\begin{proposition}[Formal Wasserstein gradient flow of Sinkhorn divergence]
\label{prop:sinkhorn_wgf_general}
Let $p:=p_{\mathrm{data}}\in \mathcal{P}_2(\mathbb{R}^d)$ be fixed, let $\tau>0$, and define
\[
\mathcal F(q):=S_\tau(p,q)
=
\mathrm{OT}_\tau(p,q)-\frac12 \mathrm{OT}_\tau(q,q)-\frac12 \mathrm{OT}_\tau(p,p),
\qquad q\in\mathcal P_2(\mathbb R^d).
\]
Assume that $q_t\in \mathcal P_2(\mathbb R^d)$ evolves according to the Wasserstein gradient flow
\begin{equation}
\label{eq:wgf_formal_general}
\partial_t q_t + \nabla\cdot(q_t v_t)=0,
\qquad
v_t(x)=-\nabla_x \frac{\delta \mathcal F}{\delta q}(q_t)(x),
\end{equation}
in a formal sense. And $q_t\in P_2(\mathbb{R}^d),\forall t$. 

Then the velocity field in \eqref{eq:wgf_formal_general} can be written as
\begin{align}
v_t(x)
&=
\int_{\mathbb R^d} (y-x)\, \mathrm d\pi^\infty_{p,q_t}(y\mid x)
-
\int_{\mathbb R^d} (x'-x)\, \mathrm d\pi^\infty_{q_t,q_t}(x'\mid x)
\label{eq:v_general_conditional_a}
\\
&=
\int_{\mathbb R^d} y\, \mathrm d\pi^\infty_{p,q_t}(y\mid x)
-
\int_{\mathbb R^d} x'\, \mathrm d\pi^\infty_{q_t,q_t}(x'\mid x)=:V^\infty_{p,q_t}(x)
\label{eq:v_general_conditional_b}
\end{align}
where $\pi^\infty_{p,q_t}(y\mid x)$ denotes the conditional law of the first variable given the second variable under $\pi^\infty_{p,q_t}$, and similarly for $\pi^\infty_{q_t,q_t}(x'\mid x)$.
\end{proposition}

\begin{proof}
We give a formal derivation.

\medskip
\noindent
\textbf{Step 0: Existence of optimal transportation plans.}

Since $\tau\ge 0$ and $p,q_t\in \mathcal{P}_2(\mathbb{R}^d)$, we have for each $t\ge 0$, the converged entropic optimal couplings
\[
\pi^\infty_{p,q_t}\in \Pi(p,q_t),
\qquad
\pi^\infty_{q_t,q_t}\in \Pi(q_t,q_t)
\]
exist, and that the first variations of the corresponding entropic OT functionals admit the usual shape-derivative representation along transport perturbations. 

\textbf{Step 1: First variation of $\mathrm{OT}_\tau(p,q)$ with respect to the second marginal.}
Fix $q\in\mathcal P_2(\mathbb R^d)$ and consider
\[
\Phi(q):=\mathrm{OT}_\tau(p,q).
\]
Let $\xi:\mathbb R^d\to\mathbb R^d$ be a smooth compactly supported vector field, and define
\[
T_\varepsilon(x):=x+\varepsilon \xi(x),
\qquad
q_\varepsilon:=(T_\varepsilon)_\# q.
\]
Using the standard envelope/shape-derivative principle for converged entropic OT, one formally obtains
\begin{equation}
\label{eq:shape_derivative_otpq}
\frac{\mathrm d}{\mathrm d\varepsilon}\Big|_{\varepsilon=0}
\mathrm{OT}_\tau(p,q_\varepsilon)
=
\iint_{\mathbb R^d\times\mathbb R^d}
\big\langle \nabla_x c(y,x), \xi(x)\big\rangle
\,\mathrm d\pi^\infty_{q,p}(y,x),
\end{equation}
where we emphasize that $x$ is the variable of the second marginal $q$.

For the quadratic cost
\[
c(y,x)=\frac12\|x-y\|^2,
\]
we have
\[
\nabla_x c(y,x)=x-y.
\]
Hence \eqref{eq:shape_derivative_otpq} becomes
\[
\frac{\mathrm d}{\mathrm d\varepsilon}\Big|_{\varepsilon=0}
\mathrm{OT}_\tau(p,q_\varepsilon)
=
\iint
\langle x-y,\xi(x)\rangle
\,\mathrm d\pi^\infty_{q,p}(y,x).
\]
Disintegrating $\pi^\infty_{q,p}$ with respect to its second marginal $q$, we obtain
\[
\frac{\mathrm d}{\mathrm d\varepsilon}\Big|_{\varepsilon=0}
\mathrm{OT}_\tau(p,q_\varepsilon)
=
\int_{\mathbb R^d}
\left\langle
\int_{\mathbb R^d}(x-y)\,\mathrm d\pi^\infty_{q,p}(y\mid x),
\,\xi(x)
\right\rangle
\,\mathrm dq(x).
\]
Therefore, in the formal Wasserstein sense,
\begin{equation}
\label{eq:first_var_otpq}
-\nabla_x \frac{\delta}{\delta q}\mathrm{OT}_\tau(p,q)(x)
=
\int_{\mathbb R^d}(y-x)\,\mathrm d\pi^\infty_{q,p}(y\mid x).
\end{equation}

\medskip
\noindent
\textbf{Step 2: First variation of $\mathrm{OT}_\tau(q,q)$.}
Now consider
\[
\Psi(q):=\mathrm{OT}_\tau(q,q).
\]
Since $q$ appears in both marginals, its first variation receives two contributions: one from perturbing the first marginal and one from perturbing the second marginal.

Because the quadratic cost is symmetric and the entropic OT functional is symmetric under exchanging the two marginals, these two contributions coincide. Therefore, the total first variation of $\Psi(q)$ is twice the contribution coming from perturbing only one marginal. Using the same shape-derivative formula as in Step 1, we obtain
\begin{equation}
\label{eq:first_var_otqq}
-\nabla_x \frac{\delta}{\delta q}\mathrm{OT}_\tau(q,q)(x)
=
2\int_{\mathbb R^d}(x'-x)\,\mathrm d\pi^\infty_{q,q}(x'\mid x).
\end{equation}

\medskip
\noindent
\textbf{Step 3: First variation of the Sinkhorn divergence.}
Since $\mathrm{OT}_\tau(p,p)$ is constant with respect to $q$, we have
\[
\frac{\delta \mathcal F}{\delta q}(q)
=
\frac{\delta}{\delta q}\mathrm{OT}_\tau(p,q)
-\frac12 \frac{\delta}{\delta q}\mathrm{OT}_\tau(q,q).
\]
Taking spatial gradients and applying \eqref{eq:first_var_otpq}--\eqref{eq:first_var_otqq} yields
\begin{align*}
-\nabla_x \frac{\delta \mathcal F}{\delta q}(q)(x)
&=
\int_{\mathbb R^d}(y-x)\,\mathrm d\pi^\infty_{q,p}(y\mid x)
-
\int_{\mathbb R^d}(x'-x)\,\mathrm d\pi^\infty_{q,q}(x'\mid x).
\end{align*}
Hence, by the definition of the Wasserstein gradient flow velocity,
\[
v(x)
=
-\nabla_x \frac{\delta \mathcal F}{\delta q}(q)(x),
\]
we obtain
\[
v(x)
=
\int_{\mathbb R^d}(y-x)\,\mathrm d\pi^\infty_{q,p}(y\mid x)
-
\int_{\mathbb R^d}(x'-x)\,\mathrm d\pi^\infty_{q,q}(x'\mid x),
\]
which is \eqref{eq:v_general_conditional_a}. Expanding both terms gives
\[
\int (y-x)\,\mathrm d\pi^\infty_{q,p}(y\mid x)
-
\int (x'-x)\,\mathrm d\pi^\infty_{q,q}(x'\mid x)
=
\int y\,\mathrm d\pi^\infty_{q,p}(y\mid x)
-
\int x'\,\mathrm d\pi^\infty_{q,q}(x'\mid x),
\]
which proves \eqref{eq:v_general_conditional_b}.
\end{proof}

\begin{remark}[On the formal nature of the derivation]
The above argument should be understood as a formal Wasserstein-calculus derivation. A fully rigorous proof would require a precise differentiability theory for
\[
q\mapsto \mathrm{OT}_\tau(p,q)
\qquad\text{and}\qquad
q\mapsto \mathrm{OT}_\tau(q,q)
\]
along transport perturbations, as well as a justification of the envelope/shape-derivative step for the converged entropic optimal couplings.
\end{remark}

\begin{remark}[Finite Sinkhorn iterations]
The above derivation applies to the converged entropic optimal couplings $\pi^\infty_{q,p}$ and $\pi^\infty_{q,q}$. If one instead uses truncated Sinkhorn couplings after $l<\infty$ iterations, the corresponding plans $\pi^l$ are generally not exact minimizers. In that case, differentiating the resulting truncated objective with respect to particle locations would, in principle, involve differentiating through $\pi^l$, unless one explicitly detaches the coupling.
\end{remark}

\section{Details in Sinkhorn Drift Model}\label{sec:sinkhorn_model}
\subsection{Sinkhorn Drift field $V^l_{q,p}$ in general setting.}
In the general (non-discrete) setting, let $\pi^{(l)}_{q,p}$ denote the (possibly truncated, after $l$ Sinkhorn iterations) entropic OT plan between $p$ and $q$ (we refer \eqref{eq:pi^l} for details), and write $\pi^{(l)}_{q,p}(\cdot\mid x)$ for its disintegration/conditional distribution at $x$. 
We define the (level-$l$) Sinkhorn drift field at location $x\in\mathbb{R}^d$ by
\begin{equation}
V^l_{q,p}(x)
:= -\int_{\mathbb{R}^d} (x-y)\,\mathrm{d}\pi^{(l)}_{q,p}(y\mid x)
\; +\; \int_{\mathbb{R}^d} (x-y)\,\mathrm{d}\pi^{(l)}_{q,q}(y\mid x).
\label{eq:sinkhorn_flow_general_l}
\end{equation}
Equivalently, if $T^l_{q,p}(x):=\int y\,\mathrm{d}\pi^{(l)}_{q,p}(y\mid x)$ denotes the associated barycentric map, then $V^l_{q,p}(x)=T^l_{q,p}(x)-T^l_{q,q}(x)$.
\subsection{Proof of Proposition \ref{pro:sinkhorn_identity}}
Let $\hat{p}=\tfrac1n\sum_{j=1}^n\delta_{y^j}$ and $\hat{q}=\tfrac1n\sum_{i=1}^n\delta_{x^i}$. In the discrete setting of Proposition~\ref{pro:sinkhorn_wgf}, the level-$l$ drift evaluated at particles takes the form
\begin{equation}
V^l_{\hat{q},\hat{p}}(x^i)=\sum_{j=1}^n (n\pi_{XY}^l)_{ij}y^j-\sum_{j=1}^n (n\pi_{XX}^l)_{ij}x^j.
\label{eq:Vl_empirical}
\end{equation}
Assume now that $\hat{p}=\hat{q}$, i.e., there exists a permutation $\sigma\in\mathfrak{S}_n$ such that $y^j=x^{\sigma(j)}$ for all $j$. Let $\mathbf{P}\in\{0,1\}^{n\times n}$ be the corresponding permutation matrix so that $\mathbf{Y}=\mathbf{P}\mathbf{X}$.
Because the cost matrix between $X$ and $Y$ is just a column-permutation of the cost matrix between $X$ and itself, the Sinkhorn iterations are equivariant under this permutation; hence the truncated plan satisfies
\begin{equation}
\pi^l_{XY}=\pi^l_{XX}\mathbf{P}^\top\qquad (\text{and similarly for }l=\infty).
\label{eq:pi_perm_equiv}
\end{equation}
Plugging $\mathbf{Y}=\mathbf{P}\mathbf{X}$ and \eqref{eq:pi_perm_equiv} into \eqref{eq:Vl_empirical} yields, for every $i$,
$$
\sum_{j}(n\pi^l_{XY})_{ij}y^j
=\sum_{j}(n\pi^l_{XX}\mathbf{P}^\top)_{ij}(\mathbf{P}\mathbf{X})_j
=\sum_{j}(n\pi^l_{XX})_{ij}x^j,
$$
and therefore
\begin{equation}
V^l_{\hat{q},\hat{p}}(x^i)=0,\qquad \forall i\in\{1,\ldots,n\}.
\label{eq:Vl_zero_when_equal_empirical}
\end{equation}
In words: when the two empirical measures coincide (up to relabeling of atoms), the cross and self barycentric projections match exactly, so the ``cross minus self'' drift cancels.

\subsection{Extend Proposition~\eqref{pro:sinkhorn_identity} to general probability measures}
\label{sec:zero_drift_general}

\begin{proposition}[Zero drift when $p=q$ in the general setting]
\label{pro:zero_drift_general}
Let $p\in\mathcal{P}_2(\mathbb{R}^d)$ and let $\pi^{(l)}_{q,p}$ be the (possibly truncated) Sinkhorn coupling defined in the appendix by \eqref{eq:sinkhorn_general}--\eqref{eq:pi^l}. Consider the level-$l$ drift field $V^l_{q,p}$ defined by \eqref{eq:sinkhorn_flow_general_l}. Then, if $p=q$, for $p$-a.e. $x$, we have: 
\begin{equation}
V^l_{q,p}(x)=0.\nonumber 
\end{equation}
\end{proposition}

\begin{proof}
For finite $l\ge 1$, when $p=q$, both terms in \eqref{eq:sinkhorn_flow_general_l} are built from the same pair of marginals. In particular, the Sinkhorn iterates $f^{(l)}$ defined by \eqref{eq:sinkhorn_general} (and thus the induced coupling $\pi^{(l)}$ in \eqref{eq:pi^l}) coincide for the two problems $(p,q)=(p,p)$ and $(q,q)=(p,p)$. Hence
$\pi^{(l)}_{p,p}=\pi^{(l)}_{q,q}$ and therefore their disintegrations satisfy $\pi^{(l)}_{p,p}(\cdot\mid x)=\pi^{(l)}_{q,q}(\cdot\mid x)$ for $p$-a.e. $x$.
Plugging this identity into \eqref{eq:sinkhorn_flow_general_l} shows that the ``cross'' and ``self'' integrals are identical and cancel, yielding $V^l_{p,p}(x)=0$. 

When $l=\infty$, we can obtain the result similarly. 
\end{proof}

\section{Identity of Sinkhorn Drift}\label{sec:identity_2}
In this section we discuss several statements of the identity when the Sinkhorn Drift $V^\infty_{q,p}\equiv 0$.

\subsection{Background in Sinkhorn Divergence}
From \cite{feydy2019interpolating}, the Sinkhorn divergence satisfies the following.
\begin{proposition}[Theorem 1 in \cite{feydy2019interpolating}]\label{pro:sinkhorn_divergence}
Let $0<\tau<\infty$ and let $S_\tau(p,q)$ denote the (debiased) Sinkhorn divergence. Assume that $p$ and $q$ are supported on a compact set and that the cost $c(x,y)$ is Lipschitz (e.g., $c(x,y)=\|x-y\|$ or $c(x,y)=\|x-y\|^2$).
\begin{enumerate}
\item (Definiteness) $S_\tau(p,q)\ge 0$, and $S_\tau(p,q)=0$ if and only if $p=q$.
\item (Convexity) For fixed $p$, the map $q\mapsto S_\tau(p,q)$ is (strictly) convex; in particular, the unique minimizer of $q\mapsto S_\tau(p,q)$ is $q=p$.
\item (Weak convergence) $q_n \rightharpoonup q$ (weakly) if and only if $S_{\tau}(q_n,q)\to 0$.
\end{enumerate}
Consequently, whenever $q\mapsto S_\tau(p,q)$ is differentiable at $q$, the stationarity condition $\nabla_q S_\tau(p,q)=0$ implies $q=p$.
\end{proposition}
\begin{proof}
For statements 1,2,3 we refer to Theorem 1 in \cite{feydy2019interpolating}. We immediately obtain the final conclusion from the first two statements.  
\end{proof}

\subsection{Identity for general Sinkhorn Drift}

\begin{proof}
Let
\[
F(q):=S_\tau(p,q),
\]
where $p$ is fixed. By Proposition \ref{prop:sinkhorn_wgf_general}, we have
\[
V^\infty_{q,p}(x)=-\nabla \frac{\delta F}{\delta q}(x), q-a.e. 
\]
Since $dq(x)=\rho_{q}(x)dx$ and $\rho_q(x)>0,\forall x\in\Omega$, we have 

Since $V^\infty_{q,p}\equiv 0$, it follows that
\[
\nabla_x\frac{\delta F}{\delta q}(x)=0
\qquad \forall x\in\Omega.
\]
Combine it with the fact $\Omega$ is connected, we have $\frac{\delta F}{\delta q}(x)$ is constant on $\Omega$. That is, there exists a constant $C\in\mathbb R$ such that
\[
\frac{\delta F}{\delta q}(x)\equiv C
\qquad \text{for all }x\in\Omega.
\]

Let $r\in\mathcal P_2(\Omega)$ be arbitrary. By the first-order variation formula, the directional derivative of $F$ at $q$ along the direction $r-q$ is
\[
\int_\Omega \frac{\delta F}{\delta q}(x)\,\mathrm d(r-q)(x)
=
\int_\Omega C\,\mathrm d(r-q)(x)
=
C\int_\Omega \mathrm d(r-q)(x)=0,
\]
since both $r$ and $q$ are probability measures.

Thus $q$ is a stationary point of $F$ on $\mathcal P_2(\Omega)$. Since $F$ is strictly convex (see Proposition \ref{pro:sinkhorn_divergence}), it follows that $q$ is the unique global minimizer of $F$. Hence
\[
F(q)\le F(p)=S_\tau(p,p)=0.
\]

On the other hand, the Sinkhorn divergence is nonnegative, so
\[
F(q)=S_\tau(p,q)\ge0.
\]

Therefore $S_\tau(p,q)=0$, and by the identity of indiscernibles of the Sinkhorn divergence we conclude that
\[
p=q.
\]

This completes the proof.
\end{proof}

\subsection{Proof of Proposition~\ref{pro:identity_tau=0} for $\tau=0$}
\begin{proof}

When $\tau=0$, the Sinkhorn divergence reduces to the classical optimal transport cost.
In this case, the optimal coupling $\pi_{q,p}^\infty$ is an $n\times n$ permutation matrix,
and $\pi_{p,p}^\infty = I_n$.

The stationarity condition becomes
$$
V^\infty_{q,p}(X)
=
n \pi_{q,p}^\infty Y - X
=0.
$$
Since $\pi_{q,p}^\infty$ is a permutation matrix,
this implies that $Y$ is a permutation of $X$.
Therefore $q=p$.
\end{proof}
\subsection{Stationary point in empirical distribution manifold.}
In this subsection we discuss the case $0<\tau<\infty$. We show that if $V^\infty_{q_X,p}=0$, then $q_X$ is a stationary point of the functional $F(q)=S_\tau(q,p)$ on $\mathcal{M}_n$. The formal result is stated in Proposition~\ref{pro:stationary}. This indicates that (restricted) convexity of $\mathcal{F}$ on $\mathcal{M}_n$ would directly imply the identity induced by the Sinkhorn drift $V^\infty$; we leave a systematic study of this aspect to future work.

\subsubsection{Background: The Wasserstein Space and Empirical Measures.}
The concept of the Wasserstein Space is rooted in the Theory of Optimal Transport \cite{villani2009optimal}. In modern analysis, the $L^2$-Wasserstein space, denoted by $\mathcal{P}_2(\mathbb{R}^d)$, is the space of probability measures with finite second moments, equipped with the metric $W_2$ that induces a Riemannian manifold \cite{otto2000generalization}. 

For applications in machine learning and generative modeling, we focus on the manifold of empirical distributions. Let $\mathcal{M}_n \subset \mathcal{P}_2(\mathbb{R}^d)$ be the set of empirical measures with $n$ Dirac masses and uniform weights:
$$
\mathcal{M}_n = \left\{ q = \frac{1}{n} \sum_{i=1}^n \delta_{x_i} : x_i \in \mathbb{R}^d \right\}.
$$
This set $\mathcal{M}_n$ can be viewed as an $nd$-dimensional submanifold (or more precisely, a quotient space $(\mathbb{R}^d)^n / \mathcal{S}_n$ where $\mathcal{S}_n$ is the symmetric group of permutations) embedded within the infinite-dimensional Wasserstein space.

At a point $q \in \mathcal{M}_n$, the tangent space $T_q \mathcal{M}_n$ consists of all infinitesimal perturbations of the measure that remain within $\mathcal{M}_n$. In the Benamou-Brenier fluid dynamics view, these perturbations are represented by velocity vectors $v_i \in \mathbb{R}^d$ assigned to each particle $x_i$, describing how the support of the distribution shifts.

\subsubsection{Surjectivity of the Configuration-to-Measure Map.}
Consider the configuration to measure map
\begin{equation}
\Phi:(\mathbb{R}^d)^n\to\mathcal{M}_n,\qquad \Phi(x_1,\dots,x_n)=\frac{1}{n}\sum_{i=1}^n\delta_{x_i}.\label{eq:figure_to_measure}
\end{equation}
The \emph{differential} (derivative) of $\Phi$ at $X=(x_1,\dots,x_n)$ is the linear map $\mathrm{d}\Phi_X:\R^{nd}\to T_{\Phi(X)}\mathcal{M}_n$ defined by the first-order expansion
\begin{equation}
\Phi(X+\varepsilon\,\delta X)=\Phi(X)+\varepsilon\,\mathrm{d}\Phi_X(\delta X)+o(\varepsilon),\qquad \varepsilon\to 0,\label{eq:dPhi}
\end{equation}
where $\delta X=(v_1,\dots,v_n)\in(\mathbb{R}^d)^n$ is an infinitesimal displacement (velocity field on particles). In particular, for empirical measures one may write the induced variation as the distribution
\begin{equation}
\mathrm{d}\Phi_X(\delta X)
= -\frac{1}{n}\sum_{i=1}^n \nabla\cdot\bigl(\delta_{x_i}\,v_i\bigr).\nonumber
\end{equation}
We say that $\mathrm{d}\Phi_X$ is an \emph{isomorphism} if it is a bijective linear map; equivalently, every tangent direction in $T_{\Phi(X)}\mathcal{M}_n$ is realized by a unique particle displacement $\delta X$.

Now we introduce the following lemma, which establishes that as long as the particles do not collapse, the coordinate degrees of freedom perfectly "span" the possible variations in the measure space.
\begin{lemma}[Non-degenerate tangent mapping]\label{lem:nondegenerate_tangent_mapping}
Let $X=(x_1,\dots,x_n)\in\R^{nd}$ be the configuration of $n$ particles, and define the parameterization map $\Phi:(\mathbb{R}^d)^n\to\mathcal{M}_n$ by \eqref{eq:figure_to_measure}.

If $X$ is non-degenerate, i.e., $x_i\neq x_j$ for all $i\neq j$, then the differential
\begin{equation}
\mathrm{d}\Phi_X: \R^{nd} \to T_{\Phi(X)}\mathcal{M}_n\nonumber
\end{equation}
is an isomorphism. Consequently, any infinitesimal displacement of $q=\Phi(X)$ within $\mathcal{M}_n$ can be uniquely represented by a coordinate displacement $\delta X\in\R^{nd}$.
\end{lemma}

\begin{proof}
Because the support points $x_i$ are distinct, the Dirac masses $\delta_{x_i}$ have disjoint supports. An infinitesimal movement of $x_i$ along a velocity $v_i\in\mathbb{R}^d$ induces a first-order variation
\begin{equation}
\delta q_i 
\;=\; -\frac{1}{n}\,\mathrm{div}\bigl(\delta_{x_i}\, v_i\bigr).\nonumber
\end{equation}
Disjointness of supports yields linear independence (in the sense of distributions) of $\{\delta q_i\}_{i=1}^n$. Since $\dim((\mathbb{R}^d)^n)=nd$ matches $\dim(T_q\mathcal{M}_n)$, the map $\mathrm{d}\Phi_X$ has trivial kernel at non-degenerate configurations, hence it is an isomorphism.
\end{proof}

\begin{proposition}\label{pro:stationary}
Under Remark~\ref{cond:identity}, suppose $0<\tau<\infty$. Then $q$ is a stationary point of $F$ restricted to the submanifold $\mathcal{M}_n$.
\end{proposition}
\begin{proof}
Let $q = \Phi(X) = \frac1n \sum_{i=1}^n \delta_{x_i}$ denote the embedding of particle configurations into the space of measures.
Define $E(X) := F(q_X)$ with $F(q) = S_\tau(q,p)$.

By the chain rule for functionals defined on manifolds, the gradient of $E$ satisfies
\begin{equation}
\langle \nabla_X E, \delta X \rangle_{\R^{nd}}
=
\left\langle 
\frac{\delta F}{\delta q},
\, \mathrm{d}\Phi_X(\delta X)
\right\rangle.
\label{eq:chain_rule_revised}
\end{equation}

Assume $\nabla_X E(X)=0$.
Then the left-hand side of \eqref{eq:chain_rule_revised} vanishes for all $\delta X \in \R^{nd}$.
Since $X$ is non-degenerate, Lemma~\ref{lem:nondegenerate_tangent_mapping} implies that 
$\mathrm{d}\Phi_X$ is an isomorphism onto the tangent space 
$T_q \mathcal{M}_n$.
Therefore for every $\delta q \in T_q \mathcal{M}_n$ there exists $\delta X$ such that
$\delta q = \mathrm{d}\Phi_X(\delta X)$.
Consequently,
\begin{equation}
\left\langle \frac{\delta F}{\delta q}, \delta q \right\rangle = 0,
\qquad
\forall\, \delta q \in T_q \mathcal{M}_n.\nonumber
\end{equation}

Hence $q$ is a stationary point of $F$ restricted to the \textbf{empirical manifold} $\mathcal{M}_n$, which is a sub-manifold of $\mathcal{P}_2(\mathbb{R}^d)$. 
\end{proof}

\subsection{Identity for $n=2$.}
We continue to discuss the case $\tau>0$ under the condition \eqref{cond:identity}. In this section, we will show if $n=2$, $V^\infty_{q,p}=0$ implies $p=q$. 

We first introduce some intermediate results: 

\begin{lemma}[Equal Means under $V^\infty_{q,p} = 0$]
\label{lem:mean}
If $V^\infty_{q,p}(X)=n\pi_{XY}Y-n\pi_{XX}X = 0$ 
then
$$
\bar{x} := \frac{1}{n}\sum_{i=1}^n x_i \;=\; \frac{1}{n}\sum_{j=1}^n y_j =: \bar{y}.
$$
\end{lemma}

\begin{proof}
Since $\pi_{XY}$ has marginals $q_X=q=\frac{1}{n}\bold 1$ and $p_Y:=p=\frac{1}{n}\bold 1$, summing over $i$ gives
\begin{align}
&\frac{1}{n}\sum_{i=1}^n x_i=\frac{1}{n}\sum_{i,j}n(\pi_{XX})_{i,j}x_i\nonumber\\
&\frac{1}{n}\sum_{i}y_i=\frac{1}{n}\sum_{i,j}n(\pi_{XY})_{i,j}y_i\nonumber. 
\end{align}
Averaging the condition $(V_{q,p}^\infty)_i = 0$ over $i$ yields $\bar{y} = \bar{x}$.
\end{proof}

\begin{lemma}[Symmetric Sinkhorn Scalings]
\label{lem:symmetric_scaling}
Let $K = \begin{pmatrix} \kappa_1 & \kappa_2 \\ \kappa_2 & \kappa_1 \end{pmatrix}$
with $\kappa_1, \kappa_2 > 0$, and let $\pi^* = \mathrm{diag}(a)\, K\, \mathrm{diag}(b)$
be the unique optimal plan of
\[
  \min_{\pi \geq 0,\; \pi \mathbf{1} = \frac{1}{2}\mathbf{1},\; \pi^\top \mathbf{1} = \frac{1}{2}\mathbf{1}}
  \sum_{ij} \pi_{ij} c_{ij} - \tau H(\pi),
\]
where $c_{ij}=\|x_i-y_j\|^2 = -\tau \log K_{ij}$ and $a, b \in \mathbb{R}^2_{>0}$.
Suppose $c$ (and thus $K$) is symmetric, $c_{1,2}=c_{2,1},c_{1,1}=c_{2,2}$,  then $a_1 = a_2$ and $b_1 = b_2$. 
\end{lemma}

\begin{proof}
Let $P = \begin{pmatrix} 0 & 1 \\ 1 & 0 \end{pmatrix}$ denote the swap permutation matrix.
Define
\[
  \tilde{\pi} := P \pi^* P^\top,
\]
i.e.\ $\tilde{\pi}_{ij} = \pi^*_{\sigma(i)\sigma(j)}$ where $\sigma$ swaps $1 \leftrightarrow 2$.

\medskip
\noindent\textbf{Step 1: $\tilde{\pi}$ is feasible.}
Since $P$ is a permutation matrix, permuting rows and columns preserves marginals:
\[
  \tilde{\pi}\,\mathbf{1} = P\pi^* P^\top \mathbf{1} = P\pi^*\mathbf{1} = P\cdot\tfrac{1}{2}\mathbf{1} = \tfrac{1}{2}\mathbf{1},
\]
and similarly $\tilde{\pi}^\top \mathbf{1} = \tfrac{1}{2}\mathbf{1}$.
Hence $\tilde{\pi}$ is feasible.

\medskip
\noindent\textbf{Step 2: $\tilde{\pi}$ is optimal.}
The objective evaluated at $\tilde{\pi}$ satisfies
\[
  \sum_{ij}\tilde{\pi}_{ij}c_{ij} - \tau H(\tilde{\pi})
  = \sum_{ij}\pi^*_{\sigma(i)\sigma(j)}\,c_{ij} - \tau H(\pi^*).
\]
By the symmetry $K_{\sigma(i)\sigma(j)} = K_{ij}$, equivalently $c_{\sigma(i)\sigma(j)} = c_{ij}$, we have
\[
  \sum_{ij}\pi^*_{\sigma(i)\sigma(j)}\,c_{ij}
  = \sum_{ij}\pi^*_{\sigma(i)\sigma(j)}\,c_{\sigma(i)\sigma(j)}
  = \sum_{ij}\pi^*_{ij}\,c_{ij}.
\]
Therefore $\tilde{\pi}$ achieves the same objective value as $\pi^*$, so $\tilde{\pi}$ is also optimal.

\medskip
\noindent\textbf{Step 3: Uniqueness forces $\tilde{\pi} = \pi^*$.}
The entropic objective is \emph{strictly} convex in $\pi$, so the optimal plan is unique.
Since both $\tilde{\pi}$ and $\pi^*$ are optimal and feasible, we conclude
\[
  \tilde{\pi} = \pi^*, \qquad \text{i.e.,}\quad
  \pi^*_{\sigma(i)\sigma(j)} = \pi^*_{ij} \quad \forall\, i,j.
\]

\medskip
\noindent\textbf{Step 4: Conclude $a_1 = a_2$ and $b_1 = b_2$.}
Writing $\pi^*_{ij} = a_i K_{ij} b_j$ and using $\pi^*_{\sigma(i)\sigma(j)} = \pi^*_{ij}$:
\[
  a_{\sigma(i)} K_{\sigma(i)\sigma(j)} b_{\sigma(j)} = a_i K_{ij} b_j.
\]
Since $K_{\sigma(i)\sigma(j)} = K_{ij}$, this simplifies to
\[
  a_{\sigma(i)} b_{\sigma(j)} = a_i b_j \quad \forall\, i, j.
\]
Setting $(i,j) = (1,1)$: $a_2 b_2 = a_1 b_1$.
Setting $(i,j) = (1,2)$: $a_2 b_1 = a_1 b_2$.
Dividing these two equations:
\[
  \frac{b_2}{b_1} = \frac{b_1}{b_2} \implies b_1^2 = b_2^2 \implies b_1 = b_2,
\]
since $b_1, b_2 > 0$. Substituting back gives $a_1 = a_2$.
\end{proof}

\begin{proof}[Proof of Proposition \ref{pro:identity_n=2}]
By Lemma~\ref{lem:mean}, $\bar{x} = \bar{y}$; translate so that $\bar{x} = \bar{y} = 0$.
Write
$$
  x_1 = -r\hat{a},\quad x_2 = r\hat{a}, \qquad
  y_1 = -s\hat{b},\quad y_2 = s\hat{b},
$$
with $r, s > 0$ and unit vectors $\hat{a}, \hat{b} \in \mathbb{R}^d$.

\medskip
\noindent\textbf{Step 1: Sinkhorn solution.}
For $n=2$, both transport matrices are $2\times 2$ doubly stochastic.
The entropic plan is
$$
  \pi^{XY}_{11} = \pi^{XY}_{22} =: \frac{\alpha}{2}, \qquad
  \pi^{XY}_{12} = \pi^{XY}_{21} =: \frac{1-\alpha}{2}. 
$$
We have the sinkhorn solution admits form 
$$[a_i K_{ij}b_j]_{i,j}, K_{i,j}=e^{\frac{\|x_i-y_j\|^2}{\tau}}.$$
By symmetry
\begin{align}
\begin{cases}
\|x_1 - y_1\|^2=\|x_2-y_2\| = r^2 + s^2 - 2rs(\hat{a}\cdot\hat{b})\\
\|x_1 - y_2\|^2 =\|x_2-y_1\|= r^2 + s^2 + 2rs(\hat{a}\cdot\hat{b})    
\end{cases}\label{eq:c11_12}
\end{align}
we can apply lemma \ref{lem:symmetric_scaling} and obtain: 
$$a_1=a_2=:a>0,b_1=b_2=:b>0.$$ 
Thus we have: 
\begin{align}
\pi^{XY}_{11}+\pi^{XY}_{12}=ab(K_{11}+K_{12})=\frac{1}{2}\nonumber 
\end{align}
and it implies: 
\begin{align}
\alpha&:=2\pi_{11}^{XY}\nonumber\\
&=aK_{11}b\nonumber\\
&=\frac{K_{11}}{K_{11}+K_{12}}\nonumber\\
&=\frac{e^{-\|x_1-y_1\|^2/\tau}}{e^{-\|x_1-y_1\|^2/\tau}+e^{-\|x_1-y_2\|^2/\tau}}\nonumber \\
&=\frac{1}{1+e^{-4rs \hat{a}\cdot\hat{b}}}\nonumber &\text{by \eqref{eq:c11_12}}\nonumber\\
&=\sigma(\frac{4rs\hat{a}\cdot\hat{b}}{\tau})\nonumber 
\end{align}
Therefore 
\begin{align}
2\alpha-1= 2\alpha - 1 = \tanh\!\left(\frac{rs\,\hat{a}\cdot\hat{b}}{\tau}\right)\nonumber 
\end{align}
where $\sigma(t)=\frac{e^t}{1+e^t}$. 
Similalry, for the self-plan, we have: 
$$
\beta=\sigma(r^2/\tau), 2\beta - 1 = \tanh\!\left(\frac{r^2}{\tau}\right) > 0.
$$

\medskip
\noindent\textbf{Step 2: Condition $V^{\infty}_{q,p}(x_1) = 0$.}
The barycentric projections at $x_1 = -r\hat{a}$ are
\begin{align}
&T_{XY}(x_1):=2 \pi_{11}^{XY}y_1+ 2\pi_{12}^{XY}y_2= -(2\alpha-1)\,s\hat{b}\nonumber\\
&T_{XX}(x_1):=2\pi_{11}^{XX}x_1+2\pi_{12}^{XX}x_2 = -(2\beta-1)\,r\hat{a}.\nonumber 
\end{align}

Setting $V^{\infty}_{q,p}(x_1) = 0$:
\begin{equation}
  \label{eq:V1}
  (2\alpha - 1)\, s\, \hat{b} = (2\beta - 1)\, r\, \hat{a}.
\end{equation}
Since $2\beta - 1 > 0$ and $r > 0$, the right-hand side is a nonzero vector
parallel to $\hat{a}$, so $\hat{b} \parallel \hat{a}$, i.e.\ $\hat{b} = \pm\hat{a}$.

\medskip
\noindent\textbf{Case 2.1: Case $\hat{b} = \hat{a}$.}
Then $\hat{a}\cdot\hat{b} = 1$ and \eqref{eq:V1} reduces to
$$
  s\tanh\!\left(\frac{rs}{\tau}\right) = r\tanh\!\left(\frac{r^2}{\tau}\right).
$$
The function $f(t) = t\,\tanh(rt/\tau)$ satisfies
$$
  f'(t) = \tanh\!\left(\tfrac{rt}{\tau}\right)
  + \tfrac{rt}{\tau}\operatorname{sech}^2\!\!\left(\tfrac{rt}{\tau}\right) > 0,
$$
so $f$ is strictly increasing. Hence $f(s) = f(r)$ implies $s = r$,
giving $y_1 = x_1,\; y_2 = x_2$.

\medskip
\noindent\textbf{Step 2.2: Case $\hat{b} = -\hat{a}$.}
Then $\hat{a}\cdot\hat{b} = -1$, so $2\alpha - 1 = \tanh(-rs/\tau) < 0$. 
Equation~\eqref{eq:V1} becomes 
$$ -(2\alpha -1) s \hat{a}=(2\beta-1)r \hat{a}, 
$$
That is 
$$s \tanh (\frac{2rs}{\tau})=r \tanh (\frac{2r^2}{\tau}).$$
and by strictly monotonicity of function $t \tanh (\frac{2rt}{\tau})$ on $(0,\infty)$, we have $s=r$.  That is 
$$y_1=x_2,y_2=x_1. $$

In both cases, we have $X$ is a permutation of $Y$ and we complete the proof. 
\end{proof}

\begin{remark}
The assumption $x_1 \neq x_2$ (and $y_1 \neq y_2$) is essential.
If $x_1 = x_2$, the self-plan $\pi^{XX}$ is non-unique and one can
construct counterexamples with $p_X \neq p_Y$ satisfying $V^{\infty}_{q,p} = 0$.
\end{remark}

\section{Sample complexity of the Sinkhorn Drift Field}\label{sec:sample_complexity}
\subsection{Sample Complexity of the Sinkhorn Barycentric Projection}

We study the problem of estimating the optimal transport map $T_0$ between
two distributions $P$ and $Q$ over $\mathbb{R}^d$ from i.i.d.\ samples
$x_1, \ldots, x_n \sim p$ and $y_1, \ldots, y_n \sim q$, using an estimator
based on entropic regularization. Our analysis closely follows
\cite{pooladian2021entropic}, adapting their framework to our notation.

\paragraph{Notation and assumptions.}
Let $\tau > 0$ be a regularization parameter.
For two probability measures $p$ and $q$ on $\mathbb{R}^d$, denote by
$\pi^{q,p}$ the optimal solution to the entropically regularized optimal
transport problem $\mathrm{OT}_\tau(p,q)$.
We define the \emph{barycentric projection} of $\pi^{q,p}$ as
\begin{equation}
    T_\tau^{q,p}(x) := \mathbb{E}_{\pi^{q,p}}[Y \mid X = x]
    = \int y \, d(\pi^{q,p})(y|x),
    \label{eq:bary}
\end{equation}
where $(\pi^{q,p})(\cdot|x)$ denotes the conditional distribution of $Y$
given $X = x$ under $\pi^{q,p}$.
Let $\hat{p}_n := \frac{1}{n}\sum_{i=1}^n \delta_{y_i}$ and
$\hat{q}_n := \frac{1}{n}\sum_{i=1}^n \delta_{x_i}$ denote the empirical
measures of $p$ and $q$, respectively. Our estimator of $T_0$ is
$T_\tau^{\hat{p}_n, \hat{q}_n}$, i.e., the barycentric projection of the
optimal entropic plan between the two empirical measures.

\paragraph{Assumptions.} We work under the following regularity conditions, 
which are identical to Assumptions (A1)--(A3) in \cite{pooladian2021entropic}:
\begin{enumerate}[label=(A\arabic*)]
    \item \label{ass:A1} $p, q \in \mathcal{P}^{\mathrm{ac}}(\Omega)$ for a 
    compact set $\Omega \subset \mathbb{R}^d$, with densities $f_p$ and $f_q$ 
    satisfying $f_p(x), f_q(x) \leq M$ and $f_q(x) \geq m > 0$ 
    for all $x \in \Omega$.
    
    \item \label{ass:A2} $\phi_0 \in C^2(\Omega)$ and 
    $\phi^*_0 \in C^{\alpha+1}(\Omega)$ for some $\alpha > 1$, where 
    $\phi_0$ is the Brenier potential satisfying $T_0 = \nabla \phi_0$.
    
    \item \label{ass:A3} The Hessian of $\phi_0$ satisfies 
    $\mu I \preceq \nabla^2 \phi_0(x) \preceq L I$ for all $x \in \Omega$, 
    for some constants $\mu, L > 0$.
    \item\label{ass:A4}
    There exist constants $c_0, \tau_0 > 0$ such that for all $\tau \leq \tau_0$,
    \begin{equation}
        \nabla^2 g_\tau(y) \preceq (1-c_0)I \quad \forall\, y \in \Omega.\nonumber
    \end{equation}
\end{enumerate}

We first introduce the following trivial result: 
\begin{proposition}\label{pro:entropic_T_sample}
Under the assumption of \ref{ass:A1}-\ref{ass:A3}, we have
\[
\mathbb{E}\,\|T^{p,q}_0-T_\tau^{\hat{p}_n,\hat{q}_n}\|\lesssim \tau^{1-d/2}\log(n)\,n^{-1/2}+O(\tau).
\]
\end{proposition}
\begin{proof}
By Theorem~5 and Corollary 1 in \cite{pooladian2021entropic} , we have
\begin{align}
&\mathbb{E}\,\|T^{q,\hat{p}}_0-T^{\hat{q},\hat{p}}_\tau\|^2_{L^2(p)}\lesssim \tau^{1-d/2}\log(n)\,n^{-1/2}.\nonumber\\
&\mathbb{E}\|T^{q,p}_\tau-T^{q,p}_0\|^2_{L^2(p)}\leq \tau^2 I_0(p,q)+\tau^{(\alpha\wedge 3+1)/2}=O(\tau)\nonumber 
\end{align}
where $I_0(p,q)$ is the Fisher information between $p,q$ along the Wasserstein geodesic. Under the assumptions \ref{ass:A1}-\ref{ass:A3} and if $\alpha\ge 2$, we have $I_0(p,q)\leq C$ for some constant $C>0$.

Combining this with the previous proposition yields
\begin{align}
&\mathbb{E}\left[\|T^{p,q}_0-T^{\hat{q},\hat{p}}_0\|_{L(p)}^2\right]\nonumber\\
&\lesssim \mathbb{E}\left[\|T_{q,p}-T_{q,\hat{p}}\|_{L(p)}^2\right]
+\mathbb{E}\left[\|T_{q,\hat{p}}-T_{\hat{q},\hat{p}}\|_{L(p)}^2\right]\nonumber\\
&\lesssim \tau^{1-d/2}\log(n)\,n^{-1/2}+O(\tau).\nonumber
\end{align}
\end{proof}

The above statement results in $\tau$ being sufficiently small. To relax this requirement, we will extend the proof techniques in \cite{pooladian2021entropic} into the new settings. 

\subsection{One-sample complexity}
\begin{proposition}[One-sample bound]
\label{prop:one-sample}
Under Assumptions \ref{ass:A1}--\ref{ass:A4}, for $\tau \leq \tau_0$ 
where $\tau_0 > 0$ is a sufficiently small constant, the entropic map 
$T_\tau^{p, \hat{q}_n}$ satisfies
\begin{equation}
    \mathbb{E}\left\|T_\tau^{q,\hat{p}_n}(x) - T_\tau^{q,p}(x)\right\|^2_{L^2(p)} 
    \lesssim 
    \tau^{1-d/2} \log(n)\, n^{-1/2},
\end{equation}
where the expectation is taken over the samples $y_1, \ldots, y_n \sim p$.
\end{proposition}

\begin{proof}
\textbf{Step 1: Reduction to a variational problem.}

We begin by expressing the squared $L^2(q)$ norm via a variational 
representation. For any $a > 0$ and vectors $u, v \in \mathbb{R}^d$, 
the following algebraic identity holds:
\begin{equation}
    \|v\|^2 = \sup_{u \in \mathbb{R}^d} 
    \left[ 4a\, u^\top v - 4a^2 \|u\|^2 \right],\forall a>0,
    \label{eq:var-rep}
\end{equation}
where the supremum is attained at $u^* = v/(2a)$. Applying this 
identity point-wise with $v = T_\tau^{q,\hat{p}_n}(x) - T_\tau^{q,p}(x)$ 
and integrating over $q$, we obtain
\begin{equation}
    \left\|T_\tau^{q,\hat{p}_n} - T_\tau^{q,p}\right\|^2_{L^2(q)} 
    = \sup_{h:\mathbb{R}^d \to \mathbb{R}^d} 
    4a \int h(x)^\top 
    \left(T_\tau^{q,\hat{p}_n}(x) - T_\tau^{q,p}(x)\right) dp(x) 
    - 4a^2\|h\|^2_{L^2(q)}.
    \label{eq:var-L2}
\end{equation}
Since $\pi^{q,\hat{p}_n}$ has first marginal $q$ and satisfies 
$T_\tau^{q,\hat{p}_n}(x) = \int y\, d(\pi^{q,\hat{p}_n})(y|x)$, 
we can write
\begin{align}
    \int h(x)^\top 
    \left(T_\tau^{q,\hat{p}_n}(x) - T_\tau^{q,p}(x)\right) dq(x)
    &= \iint h(x)^\top (y - T_\tau^{q,p}(x))\, 
    d\pi^{q,\hat{p}_n}(x,y).
    \label{eq:marginal}
\end{align}
We now define the test function
\begin{equation}
    \chi(x,y) := h(x)^\top(y - T_\tau^{q,p}(x)) - a\|h(x)\|^2,
    \label{eq:chi}
\end{equation}
so that \eqref{eq:marginal} gives
\begin{equation}
    \iint h(x)^\top(y - T_\tau^{q,p}(x))\, d\pi^{q,\hat{p}_n}(x,y)
    = \iint \chi\, d\pi^{q,\hat{p}_n} + a\|h\|^2_{L^2(q)}.
\end{equation}
Substituting back into \eqref{eq:var-L2} and simplifying, we arrive at
\begin{equation}
    \boxed{
    \left\|T_\tau^{q,\hat{p}_n} - T_\tau^{q,p}\right\|^2_{L^2(q)} 
    = 4a \sup_{h:\mathbb{R}^d \to \mathbb{R}^d} 
    \iint \chi(x,y)\, d\pi^{q,\hat{p}_n}(x,y).
    }
    \label{eq:step1-final}
\end{equation}
It therefore suffices to find an upper bound for 
$\sup_h \iint \chi\, d\pi^{q,\hat{p}_n}$.

\textbf{Step 2: Upper bound via entropic duality.}

Let $(f_\tau, g_\tau)$ be the optimal entropic potentials for $(q, p)$,
normalized to satisfy the dual optimality conditions
\begin{equation}
    \int e^{(f_\tau(x) + g_\tau(y) - \frac{1}{2}\|x-y\|^2)/\tau}
    \, dp(y) = 1 \quad \forall x \in \mathbb{R}^d,
    \label{eq:dual-opt-x}
\end{equation}
\begin{equation}
    \int e^{(f_\tau(x) + g_\tau(y) - \frac{1}{2}\|x-y\|^2)/\tau}
    \, dq(x) = 1 \quad \forall y \in \mathbb{R}^d.
    \label{eq:dual-opt-y}
\end{equation}
The $q \otimes p$ density of $\pi^{q,p}$ is then given by
\begin{equation}
    \tilde{\pi}^{q,p}(x,y) := 
    e^{(f_\tau(x) + g_\tau(y) - \frac{1}{2}\|x-y\|^2)/\tau}.
    \label{eq:density}
\end{equation}
We apply Proposition 1 of \cite{pooladian2021entropic} to  $\pi^{q,\hat{p}_n}$ with the choice
\begin{equation}
    \eta(x,y) := \tau\chi(x,y) + f_\tau(x) + g_\tau(y),\nonumber 
\end{equation}

\begin{align}
OT_\tau(q,\hat{p}_n)&=\sup_{\eta\in L^1(\pi^{q,\hat{p}_n})} \int \eta d\pi^{q,\hat{p}_n}-\tau \iint e^{(\eta(x,y)-1/2\|x-y\|^2)/\tau}d\hat{p}_n(y)dq(x)+\tau \nonumber\\
&\ge \tau\iint \chi\, d\pi^{q,\hat{p}_n} 
    + \int f_\tau\, dq + \int g_\tau\, d\hat{p}_n
    - \tau\iint e^{\chi}\tilde{\pi}^{q,p}\, dq\, d\hat{p}_n 
    + \tau\nonumber\\
&=\tau \iint \chi d\pi^{q,\hat{p}_n}+OT_\tau(q,p)+\int g_\tau (d\hat{p}_n-dp)+\tau \label{eq:prop1-applied}
\end{align}
where the last line follows from the dual identity $\mathrm{OT}_\tau(q,p) = \int f_\tau\, dq 
+ \int g_\tau\, dp$. 
Dividing \eqref{eq:prop1-applied} by $\tau$ and rearranging yields
\begin{align}
\iint \chi\, d\pi^{q,\hat{p}_n} 
    \leq 
    &\underbrace{\iint e^{\chi}\tilde{\pi}^{q,p}\, dq\, d\hat{p}_n - 1}
    _{\text{Term A}}
    + \underbrace{\tau^{-1}
    \left[\mathrm{OT}_\tau(q,\hat{p}_n) - \mathrm{OT}_\tau(q,p)\right]}
    _{\Delta_1}\nonumber\\
    &+ \underbrace{\tau^{-1}\int g_\tau\, d(p - \hat{p}_n)}
    _{\Delta_2}.
    \label{eq:step2-bound}
\end{align}

\textbf{Step 3: Decomposition and bound of Term A.}

We split Term A by writing $d\hat{p}_n = dp + d(\hat{p}_n - p)$:
\begin{equation}
    \iint e^{\chi}\tilde{\pi}^{q,p}\, dq\, d\hat{p}_n - 1
    = \underbrace{\iint e^{\chi}\tilde{\pi}^{q,p}\, dq\, dp - 1}
    _{\leq\, 0}
    + \underbrace{\iint e^{\chi}\tilde{\pi}^{q,p}\, dq\, d(\hat{p}_n - p)}
    _{\Delta_3}.
    \label{eq:termA-split}
\end{equation}

\textbf{The first term is non-positive.} 
For fixed $x$, the conditional measure 
$\tilde{\pi}^{q,p}(x,\cdot)\,dp(\cdot)$ is a probability measure on 
$\mathbb{R}^d$ by \eqref{eq:dual-opt-x}, with conditional mean 
$T_\tau^{q,p}(x)$. Therefore:
\begin{equation}
    \int (y - T_\tau^{q,p}(x))\, \tilde{\pi}^{q,p}(x,y)\, dp(y) = 0 
    \quad \forall x.
    \label{eq:zero-mean}
\end{equation}
Since $\Omega$ is compact, we have 
$|h(x)^\top(y - T_\tau^{q,p}(x))| \leq C\|h(x)\|$ 
for all $y \in \Omega$ and some constant $C > 0$. 
Combined with \eqref{eq:zero-mean}, Hoeffding's inequality implies 
that for $a \geq C^2/2$:
\begin{align}
&\int e^\chi \tilde{\pi}^{q,p}dqdp \nonumber\\
&=\int e^{h(x)^\top(y - T_\tau^{q,p}(x)) - a\|h(x)\|^2}\,
    \tilde{\pi}^{q,p}(x,y)\, dp(y)\nonumber\\
&\leq \mathbb{E}_{x\sim q}\left[\mathbb{E}_{y\sim p}[e^{h^\top (y-T_\tau^{q,p})(x)}|x]e^{-a\|h(x)\|^2}\right]\nonumber\\
&\leq \mathbb{E}_{x\sim q}\left[e^{(C^2-2a)\|h(x)\|^2/2}\right]\nonumber\\
&\leq 1
\label{eq:hoeffding}
\end{align}

Therefore Term A is bounded above by $\Delta_3$ alone:
\begin{equation}
    \text{Term A} \leq \Delta_3 
    = \iint e^{\chi(x,y)}\tilde{\pi}^{q,p}(x,y)\, dq(x)\, 
    d(\hat{p}_n - p)(y).
    \label{eq:termA-bound}
\end{equation}

\textbf{Bounding $\Delta_3$ in expectation.}
By the Propostition \ref{pro:delta3}, under the assumptions \ref{ass:A1}-\ref{ass:A4}, we have: 
\begin{equation}
    \mathbb{E}|\Delta_3| \lesssim \tau^{-d/2}\, n^{-1/2}.
    \label{eq:delta3-bound}
\end{equation}

\textbf{Step 4: Bounding $\Delta_1$ and $\Delta_2$.}

\textbf{Bounding $\Delta_1$.}
Recall that 
\begin{equation}
    \Delta_1 = \tau^{-1}\left[\mathrm{OT}_\tau(q,\hat{p}_n) 
    - \mathrm{OT}_\tau(q,p)\right].
\end{equation}
This term measures the fluctuation of the entropic OT value when 
the target $p$ is replaced by its empirical measure $\hat{p}_n$. 
By Corollary 3 of \cite{pooladian2021entropic}, we have: 
\begin{equation}
    \mathbb{E}\Delta_1 
    \lesssim (\tau^{-1} + \tau^{-d/2})\log(n)\,n^{-1/2}.
    \label{eq:delta1-bound}
\end{equation}

\textbf{Bounding $\Delta_2$.}
Recall that 
\begin{equation}
    \Delta_2 = \tau^{-1}\int g_\tau\, d(p - \hat{p}_n),\nonumber
\end{equation}
where $g_\tau$ is the optimal entropic potential for 
$\mathrm{OT}_\tau(q,p)$. This term is an empirical process 
indexed by the single function $g_\tau$. By Lemmas 7 and 8 of 
\cite{pooladian2021entropic}, with the same correspondence as above:
\begin{equation}
    \mathbb{E}|\Delta_2| 
    \lesssim (\tau^{-1} + \tau^{-d/2})\log(n)\,n^{-1/2}.
    \label{eq:delta2-bound}
\end{equation}

\textbf{Combining all terms.}
From \eqref{eq:termA-bound}, \eqref{eq:delta1-bound}, 
\eqref{eq:delta2-bound}:
\begin{equation}
    \mathbb{E}\sup_h\iint \chi\, d\pi^{q,\hat{p}_n}
    \leq \mathbb{E}\Delta_1 + \mathbb{E}|\Delta_2| 
    + \mathbb{E}|\Delta_3|
    \lesssim (\tau^{-1} + \tau^{-d/2})\log(n)\,n^{-1/2}.
    \label{eq:all-terms}
\end{equation}
Substituting back into \eqref{eq:step1-final} with $a = C\tau$ 
for a sufficiently large constant $C \geq L$:
\begin{align}
    \mathbb{E}\left\|T_\tau^{q,\hat{p}_n} - T_\tau^{q,p}
    \right\|^2_{L^2(q)}
    &= 4a\,\mathbb{E}\sup_h\iint\chi\,d\pi^{q,\hat{p}_n}\nonumber \\
    &\lesssim \tau \cdot 
    (\tau^{-1} + \tau^{-d/2})\log(n)\,n^{-1/2}\nonumber\\
    &= (1 + \tau^{1-d/2})\log(n)\,n^{-1/2}\nonumber\\
    &\lesssim \tau^{1-d/2}\log(n)\,n^{-1/2},\nonumber
    \label{eq:final-bound}
\end{align}
where the last step uses $\tau \leq \tau_0 \leq 1$, 
so $1 \leq \tau^{1-d/2}$ for $d \geq 2$. 
This completes the proof of Proposition \ref{prop:one-sample}.
\qed
\end{proof}

\begin{proposition}[$\Delta_3$ bound]
\label{pro:delta3}
Under Assumptions~\ref{ass:A1}--\ref{ass:A4}, take $a \geq C_\Omega^2/2$ 
where $C_\Omega = \mathrm{diam}(\Omega)$. For $\tau \leq \tau_0$,
\begin{equation}
    \mathbb{E}\sup_{h:\mathbb{R}^d\to\mathbb{R}^d}|\Delta_3(h)| 
    \lesssim \tau^{-d/2}\,n^{-1/2},
\end{equation}
where 
$\Delta_3(h) := \iint e^{\chi(x,y)}\tilde{\pi}^{q,p}(x,y)\,dq(x)\,d(\hat{p}_n-p)(y)$
and $\chi(x,y) = h(x)^\top(y - T_\tau^{q,p}(x)) - a\|h(x)\|^2$.
\end{proposition}

\begin{proof}
The proof proceeds in three parts.

\medskip
\noindent\textbf{Part 1: Pointwise upper bound on $\tilde{\pi}^{q,p}$.}

\begin{lemma}
\label{sublem:pi-bound}
Under Assumptions~\ref{ass:A1} and \ref{ass:A4}, 
\begin{equation}
    K := \sup_{x\in\mathrm{supp}(q),\,y\in\Omega}\,\tilde{\pi}^{q,p}(x,y) 
    \lesssim \tau^{-d/2},\nonumber
\end{equation}
uniformly in $x$.
\end{lemma}

\begin{proof}[Proof of Sub-lemma]
By Assumption~\ref{ass:A1}, $p(y) \geq m > 0$ for all $y \in \Omega$, so
\begin{equation}
    \tilde{\pi}^{q,p}(x,y) 
    = \frac{\tilde{\pi}^{q,p}(x,y)\cdot p(y)}{p(y)} 
    \leq \frac{1}{m}\,\tilde{\pi}^{q,p}(x,y)\cdot p(y). \nonumber 
\end{equation}
It therefore suffices to bound $\sup_{y\in\Omega}\tilde{\pi}^{q,p}(x,y)\cdot p(y)$.
By the dual optimality condition~\eqref{eq:dual-opt-x}, for each fixed $x$,
\begin{equation}
    \int \tilde{\pi}^{q,p}(x,y)\,dp(y) = 1,\nonumber 
\end{equation}
so $y\mapsto\tilde{\pi}^{q,p}(x,y)\,p(y)$ is a probability density with 
respect to Lebesgue measure $dy$. Substituting the definition 
$\tilde{\pi}^{q,p}(x,y) = e^{(f_\tau(x)+g_\tau(y)-\frac{1}{2}\|x-y\|^2)/\tau}$, 
we write
\begin{equation}
    \tilde{\pi}^{q,p}(x,y)\,p(y) 
    = \frac{e^{H_x(y)/\tau}}{c(x)},
    \qquad
    H_x(y) := g_\tau(y) - \tfrac{1}{2}\|x-y\|^2 + \tau\log p(y),
    \label{eq:Hx-def}
\end{equation}
where $c(x) := e^{-f_\tau(x)/\tau} = \int e^{H_x(y)/\tau}\,dy$ 
is the normalization constant.

\medskip
\noindent\textit{Strong concavity of $H_x$.}
Computing the Hessian of $H_x$:
\begin{equation}
    \nabla^2 H_x(y) = \nabla^2 g_\tau(y) - I + \tau\nabla^2\log p(y).\nonumber 
\end{equation}
By Assumption~\ref{ass:A4}, $\nabla^2 g_\tau(y) \preceq (1-c_0)I$, so
\begin{equation}
    \nabla^2 H_x(y) \preceq -c_0 I + \tau\nabla^2\log p(y).\nonumber 
\end{equation}
Since $p \in C^2(\Omega)$ and $\Omega$ is compact, 
$\|\nabla^2\log p\|_{L^\infty(\Omega)} \leq C_p$ for some constant $C_p > 0$.
Therefore, for $\tau \leq \tau_0$ with $\tau_0 \leq c_0/(2C_p)$,
\begin{equation}
    \nabla^2 H_x(y) \preceq -\frac{c_0}{2}I 
    \quad \forall\, y \in \Omega.
    \label{eq:Hx-concave}
\end{equation}
That is, $H_x$ is $(c_0/2)$-strongly concave on $\Omega$.

\medskip
\noindent\textit{Lower bound on $c(x)$.}
Let $\hat{y}(x) := \arg\max_{y\in\Omega}H_x(y)$.
By~\eqref{eq:Hx-concave} and Taylor's theorem,
\begin{align}
H_x(y)&=H_x(\hat{y})+\underbrace{\nabla H_x(\hat y)^\top}_{\bold{0}} (y-\hat{y})+\frac{1}{2}(y-\hat{y})^\top \nabla^2 H_x(\xi) (y-\hat{y})\nonumber\\
&\leq H_x(\hat{y}) - \frac{c_0}{4}\|y - \hat{y}\|^2 
    \quad \forall\, y \in \Omega. \nonumber 
\end{align}
Therefore,
\begin{align}
    c(x) 
    = \int_\Omega e^{H_x(y)/\tau}\,dy 
    &\geq \int_{\mathbb{R}^d} 
        e^{H_x(\hat{y})/\tau 
        - \frac{c_0}{4\tau}\|y-\hat{y}\|^2}\,dy \nonumber\\
    &= e^{H_x(\hat{y})/\tau}
        \cdot\left(\frac{4\pi\tau}{c_0}\right)^{d/2},
    \label{eq:cx-lower}
\end{align}
where we extended the domain to $\mathbb{R}^d$ and evaluated the 
resulting Gaussian integral.

\medskip
\noindent\textit{Conclusion of Sub-lemma.}
Combining the numerator bound $e^{H_x(y)/\tau} \leq e^{H_x(\hat{y})/\tau}$ 
with~\eqref{eq:cx-lower},
\begin{equation}
    \sup_{y\in\Omega}\tilde{\pi}^{q,p}(x,y)\cdot p(y) 
    = \sup_{y\in\Omega}\frac{e^{H_x(y)/\tau}}{c(x)}
    \leq \frac{e^{H_x(\hat{y})/\tau}}{c(x)}
    \leq \left(\frac{c_0}{4\pi\tau}\right)^{d/2}
    \lesssim \tau^{-d/2}.
\end{equation}
Hence $K \leq \frac{1}{m}\sup_y \tilde{\pi}^{q,p}(x,y)\cdot p(y) 
\lesssim \tau^{-d/2}$, uniformly in $x$.
\end{proof}

\medskip
\noindent\textbf{Part 2: Reduction to a fixed-$x$ empirical process.}

By Fubini's theorem,
\begin{equation}
    \Delta_3(h) 
    = \int_\Omega 
    \underbrace{
    \left[
        \int_\Omega 
        e^{h(x)^\top(y-T_\tau^{q,p}(x))-a\|h(x)\|^2}
        \tilde{\pi}^{q,p}(x,y)\,
        d(\hat{p}_n-p)(y)
    \right]
    }_{\Psi_{h(x)}(x)}
    dq(x).
\end{equation}
For each fixed $x$, the term $\Psi_{h(x)}(x)$ depends on $h$ only through 
the vector $h(x) =: v \in \mathbb{R}^d$. By the triangle inequality,
\begin{equation}
    \sup_h|\Delta_3(h)| 
    \leq \int_\Omega 
    \sup_{v\in\mathbb{R}^d}
    \left|
        \int_\Omega 
        e^{v^\top(y-T_\tau^{q,p}(x))-a\|v\|^2}
        \tilde{\pi}^{q,p}(x,y)\,
        d(\hat{p}_n-p)(y)
    \right|
    dq(x).
\end{equation}
Taking expectations, it suffices to show uniformly in $x\in\mathrm{supp}(q)$:
\begin{equation}
\label{eq:fixedx-target}
    \mathbb{E}
    \sup_{v\in\mathbb{R}^d}
    \left|
        \int_\Omega 
        e^{v^\top(y-T_\tau^{q,p}(x))-a\|v\|^2}
        \tilde{\pi}^{q,p}(x,y)\,
        d(\hat{p}_n-p)(y)
    \right|
    \lesssim \tau^{-d/2}n^{-1/2}.
\end{equation}

\medskip
\noindent\textbf{Part 3: Covering number bound and Dudley chaining.}

Fix $x\in\mathrm{supp}(q)$ and define
\begin{equation}
    \tilde{\mathcal{J}}_\tau^x 
    := \left\{
        y\mapsto e^{j_v(y)}\tilde{\pi}^{q,p}(x,y)
        \;:\; v\in\mathbb{R}^d
    \right\},
    \quad
    j_v(y) := v^\top(y-T_\tau^{q,p}(x)) - a\|v\|^2.\nonumber 
\end{equation}
Thus 
$$
\sup_h |\Delta_3(h)| \leq \int \sup_{f\in \tilde{J}_\tau^x}|\int f d(\hat{p}_n-p)|dq(x).  
$$
\textbf{Part 3.1, Bound of the cover number} 
\begin{lemma}
We claim that for any $\delta\in(0,1)$,
\begin{equation}
\label{eq:covering-number}
    \log N\!\left(
        \delta,\,\tilde{\mathcal{J}}_\tau^x,\,\|\cdot\|_{L^\infty(p)}
    \right) 
    \lesssim d\log(K/\delta).
\end{equation}    
\end{lemma}
\begin{proof}
Set $R:=\delta^{-1/2}$ and let $\mathcal{N}_\delta$ be a $\delta^{3/2}$-net 
of $B_R(0)\subset\mathbb{R}^d$ satisfying $|\mathcal{N}_\delta|\lesssim\delta^{-d}$. 
Fix an arbitrary $w$ with $\|w\|=R$, and let $\tilde{\mathcal{G}}_\delta$ be the 
set of functions in $\tilde{\mathcal{J}}_\tau^x$ corresponding to $\mathcal{N}_\delta$ 
and to $w$.

\medskip
\noindent\textit{Case 1: $\|v\|\geq R$.}
By Young's inequality,
\begin{equation}
    j_v(y) 
    \leq C_\Omega\|v\| - a\|v\|^2 
    \leq \frac{C_\Omega^2}{4a} - \frac{a}{2}\|v\|^2.\nonumber
\end{equation}
Hence $e^{j_v(y)}\tilde{\pi}^{q,p}(x,y) \leq Ke^{-\frac{a}{2}/\delta}$, 
and similarly for $w$, so
\begin{equation}
    \sup_{y\in\Omega}
    \left|e^{j_v(y)} - e^{j_w(y)}\right|
    \tilde{\pi}^{q,p}(x,y) 
    \leq 2Ke^{-(a/2)/\delta} 
    \leq K\delta,\nonumber
\end{equation}
for $\delta$ sufficiently small since $e^{-(a/2)/\delta}=o(\delta)$.

\medskip
\noindent\textit{Case 2: $\|v\|\leq R$.}
Pick $u\in\mathcal{N}_\delta$ with $\|u-v\|\leq\delta^{3/2}$. For any $y\in\Omega$,
\begin{equation}
    |j_v(y)-j_u(y)| 
    \leq \|v-u\|C_\Omega + a(\|v\|+\|u\|)\|v-u\| 
    \leq \delta^{3/2}(C_\Omega+2aR) 
    \lesssim \delta,\nonumber
\end{equation}
where the last step uses $R=\delta^{-1/2}$ and $a\leq 1$. Since 
$j_v(y)\leq C_\Omega^2/(4a)=O(1)$, the inequality $|e^s-e^t|\leq e^{s\vee t}|s-t|$ gives
\begin{equation}
    \sup_{y\in\Omega}
    \left|e^{j_v(y)}-e^{j_u(y)}\right|
    \tilde{\pi}^{q,p}(x,y) 
    \lesssim K\delta.\nonumber
\end{equation}
In both cases, every element of $\tilde{\mathcal{J}}_\tau^x$ is approximated 
by some element of $\tilde{\mathcal{G}}_\delta$ to $L^\infty(p)$-precision $O(K\delta)$. 
Since $|\tilde{\mathcal{G}}_\delta|\lesssim\delta^{-d}$, replacing $\delta$ by 
$\delta/C$ establishes~\eqref{eq:covering-number}.
\end{proof}

\medskip
\noindent\textbf{Part 3.2 Chaining Bound.}
The $L^\infty(p)$ envelope of $\tilde{\mathcal{J}}_\tau^x$ satisfies
\begin{equation}
    \sup_{v\in\mathbb{R}^d}\sup_{y\in\Omega} 
    e^{j_v(y)}\tilde{\pi}^{q,p}(x,y) 
    \leq e^{C_\Omega^2/(4a)}K 
    \lesssim K 
    \lesssim \tau^{-d/2}.\nonumber
\end{equation}
By Dudley's entropy integral bound (\cite{gine2021mathematical} Theorem~3.5.1)
\begin{equation}
    \mathbb{E}
    \sup_{f\in\tilde{\mathcal{J}}_\tau^x}
    \left|\int f\,d(\hat{p}_n-p)\right|
    \lesssim 
    n^{-1/2}
    \int_0^{K} 
    \sqrt{\log N(\delta,\tilde{\mathcal{J}}_\tau^x,L^\infty)}\,d\delta.\label{eq:bound_cont1}
\end{equation}
Substituting~\eqref{eq:covering-number} and setting $s=\delta/K$, the R.H.S. of  above inequality can be bounded: 
\begin{equation}
\eqref{eq:bound_cont1}  \lesssim Kn^{-1/2}
    \int_0^1\sqrt{-d\log s}\,ds 
    = Kn^{-1/2}\cdot\frac{\sqrt{\pi d}}{2}
    \lesssim \tau^{-d/2}n^{-1/2},\nonumber
\end{equation}
which establishes~\eqref{eq:fixedx-target} uniformly in $x$.

\medskip
\noindent\textbf{Conclusion.}
Integrating~\eqref{eq:fixedx-target} over $x$ with respect to $q$,
\begin{equation}
    \mathbb{E}\sup_h|\Delta_3(h)| 
    \leq 
    \mathbb{E}
    \sup_{f\in\tilde{\mathcal{J}}_\tau^x}
    \left|\int f\,d(\hat{p}_n-p)\right|
    \lesssim 
    \tau^{-d/2}n^{-1/2}
    \cdot
    \underbrace{\int_\Omega dq(x)}_{=1}
    =\tau^{-d/2}n^{-1/2}.\nonumber
    \qed
\end{equation}
\end{proof}

\begin{proposition}\label{pro:entropic_T_sample_2}
Unnder the assumption of \ref{ass:A1}-\ref{ass:A4}, we have
\[
\mathbb{E}\,\|T^{p,q}_\tau-T_\tau^{\hat{p}_n,\hat{q}_n}\|\lesssim \tau^{1-d/2}\log(n)\,n^{-1/2}.
\]
\end{proposition}
\begin{proof}
By Theorem~5 in \cite{pooladian2021entropic} , we have
\begin{align}
&\mathbb{E}\,\|T^{q,\hat{p}}_\tau-T^{\hat{q},\hat{p}}_\tau\|^2_{L^2(p)}\lesssim \tau^{1-d/2}\log(n)\,n^{-1/2}.\nonumber\\
\end{align}
By Proposition \ref{prop:one-sample}, we have 

\begin{align}
&\mathbb{E}\,\|T^{q,p}_\tau-T^{q,\hat{p}}_\tau\|^2_{L^2(p)}\lesssim \tau^{1-d/2}\log(n)\,n^{-1/2}.\nonumber\\
\end{align}

Combining this with the previous proposition yields
\begin{align}
&\mathbb{E}\left[\|T^{p,q}_\tau-T^{\hat{q},\hat{p}}_\tau\|_{L(p)}^2\right]\nonumber\\
&\lesssim \mathbb{E}\left[\|T^{q,p}_\tau-T^{q,\hat{p}}_\tau\|_{L(p)}^2\right]
+\mathbb{E}\left[\|T^{q,\hat{p}}_\tau-T^{\hat{q},\hat{p}}_\tau\|_{L(p)}^2\right]\nonumber\\
&\lesssim \tau^{1-d/2}\log(n)\,n^{-1/2}+O(\tau).\nonumber
\end{align}
\end{proof}

\subsection{Sample complexity of the Drift Field}
Based on Proposition~\ref{pro:entropic_T_sample}, we immediately obtain the following.

\begin{proposition}
Under the same assumptions in Proposition \ref{pro:entropic_T_sample}, 
let $\hat{p}$ and $\hat{q}$ be empirical distributions of size $n$, i.i.d.\ sampled from $p$ and $q$, respectively. Then
\begin{align}
\mathbb{E}\left[\|V_{q,p}^\infty - V^\infty_{\hat{q},\hat{p}}\|^2\right]
\lesssim \tau^{1-d/2}\log(n)\,n^{-1/2}.\nonumber
\end{align}
\end{proposition}
\begin{proof}
By definition \eqref{eq:sinkhorn_drif_general}, we have
\begin{align}
&\mathbb{E}\,\|V_{q,p}^\infty(x)-V_{\hat{q},\hat{p}}\|^2\nonumber\\
&\lesssim \mathbb{E}\,\|T_{q,p}(x)-T_{\hat{q},\hat{p}}\|^2
+\mathbb{E}\,\|T_{q,q}(x)-T_{\hat{q},\hat{q}}\|^2\nonumber\\
&\lesssim \tau^{1-d/2}\log(n)\,n^{-1/2},\nonumber
\end{align}
where the last step follows from Proposition~\ref{pro:entropic_T_sample}.
\end{proof}

\begin{remark}
At a stationary point where $V^\infty(x)\equiv 0$, the above proposition implies
\[
\mathbb{E}\,\|V_{\hat{q},\hat{p}}\|^2\lesssim \tau^{1-d/2}\log(n)\,n^{-1/2}.
\]
Intuitively, when the batch size $n$ is sufficiently large and the empirical drift $V^\infty_{\hat{q},\hat{p}}$ is sufficiently small, we can conclude with high confidence that $V^\infty_{q,p}=0$, and hence $q_\theta=q=p$.
\end{remark}

\section{Stop-gradient regression as a forward Euler discretization}
\label{subsec:sg_euler}

This subsection formalizes the connection between drifting-model updates and particle gradient flows. We show that a stop-gradient regression loss realizes an explicit forward Euler step when paired with a $q$-weighted (preconditioned) gradient descent update. Concretely, let $v_t$ denote the gradient-flow velocity and define the drift field as the velocity itself, $V_t:=v_t$. Regressing $x_t^i$ toward the stop-gradient target $x_t^i+V_t(x_t^i)$ and taking one $q$-weighted gradient step yields the forward Euler discretization of the ODE $\dot x_t^i = v_t(x_t^i)$.

\paragraph{Particle gradient flow and drift field.}
Let $X_t=\{x_t^i\}_{i=1}^n\subset\mathbb{R}^d$ denote a system of particles with weights $\{q_i\}_{i=1}^n$ satisfying $q_i>0$ and $\sum_{i=1}^n q_i=1$, and let $q_{X_t}$ be the associated empirical measure. Consider a differentiable functional $\mathcal{F}$ on measures and the corresponding particle dynamics
\begin{equation}
\dot x_t^i\;=\; v_t(x_t^i)
\;:=\; -\frac{1}{q_i}\nabla_{x_t^i}\mathcal{F}(q_{X_t}),
\qquad i=1,\ldots,n.
\label{eq:wgf_particles}
\end{equation}
Following the drift viewpoint, we define the \emph{drift field} as the velocity,
\begin{equation}
V_t(x_t^i) \;:=\; v_t(x_t^i)
\;=\; -\frac{1}{q_i}\nabla_{x_t^i}\mathcal{F}(q_{X_t}),
\qquad i=1,\ldots,n.
\label{eq:drift_def}
\end{equation}

\subsection{Sample complexity of the Drift Field}
Based on Proposition~\ref{pro:entropic_T_sample}, we immediately obtain the following.

\begin{proposition}
Under the same assumptions in Proposition \ref{pro:entropic_T_sample}, 
let $\hat{p}$ and $\hat{q}$ be empirical distributions of size $n$, i.i.d.\ sampled from $p$ and $q$, respectively. Then
\begin{align}
\mathbb{E}\left[\|V_{q,p}^\infty - V^\infty_{\hat{q},\hat{p}}\|^2\right]
\lesssim \tau^{1-d/2}\log(n)\,n^{-1/2}.\nonumber
\end{align}
\end{proposition}
\begin{proof}
By definition \eqref{eq:sinkhorn_drif_general}, we have
\begin{align}
&\mathbb{E}\,\|V_{q,p}^\infty(x)-V_{\hat{q},\hat{p}}\|^2\nonumber\\
&\lesssim \mathbb{E}\,\|T_{q,p}(x)-T_{\hat{q},\hat{p}}\|^2
+ \mathbb{E}\,\|T_{q,q}(x)-T_{\hat{q},\hat{q}}\|^2\nonumber\\
&\lesssim \tau^{1-d/2}\log(n)\,n^{-1/2}.\nonumber
\end{align}
where the last step follows from Proposition~\ref{pro:entropic_T_sample}.
\end{proof}

\begin{remark}
At a stationary point where $V^\infty(x)\equiv 0$, the above proposition implies
\[
\mathbb{E}\,\|V_{\hat{q},\hat{p}}\|^2\lesssim \tau^{1-d/2}\log(n)\,n^{-1/2}.
\]
Intuitively, when the batch size $n$ is sufficiently large and the empirical drift $V^\infty_{\hat{q},\hat{p}}$ is sufficiently small, we can conclude with high confidence that $V^\infty_{q,p}=0$, and hence $q_\theta=q=p$.
\end{remark}

\paragraph{Stop-gradient regression loss.}
Given the drift field $V_t$ evaluated at the current particle system $X_t$, define the stop-gradient regression objective
\begin{equation}
\mathcal{L}_t(X)
\;:=\;
\frac{1}{2}\sum_{i=1}^n  q_i
\big\|x^i - \text{sg}\big(x_t^i+V_t(x_t^i)\big)\big\|^2,
\label{eq:sg_loss}
\end{equation}
where $\text{sg}(\cdot)$ denotes the stop-gradient operator (treated as a constant in backpropagation). Note that the target $x_t^i+V_t(x_t^i)$ is computed from $X_t$ and is held fixed when differentiating $\mathcal{L}_t$ with respect to $X$.

\paragraph{$q$-weighted GD equals forward Euler.}
We compute the gradient of $\mathcal{L}_t$ with respect to $x^i$:
$$
\nabla_{x^i}\mathcal{L}_t(X)
=q_i\Big(x^i-\text{sg}\big(x_t^i+V_t(x_t^i)\big)\Big).
$$
Evaluating at $X=X_t$ and using that $\text{sg}(\cdot)$ does not affect forward values yields
$$
\nabla_{x^i}\mathcal{L}_t(X_t)
=q_i\Big(x_t^i-(x_t^i+ V_t(x_t^i))\Big)
=-q_iV_t(x_t^i).
$$
We now take a single $q$-weighted (preconditioned) gradient descent step with learning rate $\eta>0$,
\begin{equation}
x_{t+1}^i
\;=\;
x_t^i-\eta\,\frac{1}{q_i}\nabla_{x^i}\mathcal{L}_t(X_t)
\;=\;
x_t^i+\eta V_t(x_t^i),
\qquad i=1,\ldots,n,
\label{eq:gd_update}
\end{equation}
which is exactly the forward Euler discretization of the ODE $\dot x_t^i=V_t(x_t^i)=v_t(x_t^i)$ with step size $\eta$.


\paragraph{Remark.}
The same argument applies when the particles $x_t^i$ are outputs of a parametric generator $x_t^i=f_\theta(\epsilon^i)$: the stop-gradient regression objective induces an output-space update in the direction $V_t$, up to the usual parameterization-dependent preconditioning through the Jacobian of $f_\theta$.

\section{Additional Drift Trajectory Results for Varying $\tau$}
\label{app:drift_tau}

Figure~\ref{fig:drift_trajectories_full_appendix} shows the full trajectory grid for
the temperature sweep discussed in Section~5.1. It confirms the same qualitative
pattern as in the teaser figure: Sinkhorn remains more stable at small $\tau$,
while one-sided and two-sided normalization are more sensitive to low-temperature
degeneration. The masking variant avoids diagonal self-interaction, but also
changes the resulting dynamics.

\begin{figure}[H]
\centering
\includegraphics[width=\textwidth]{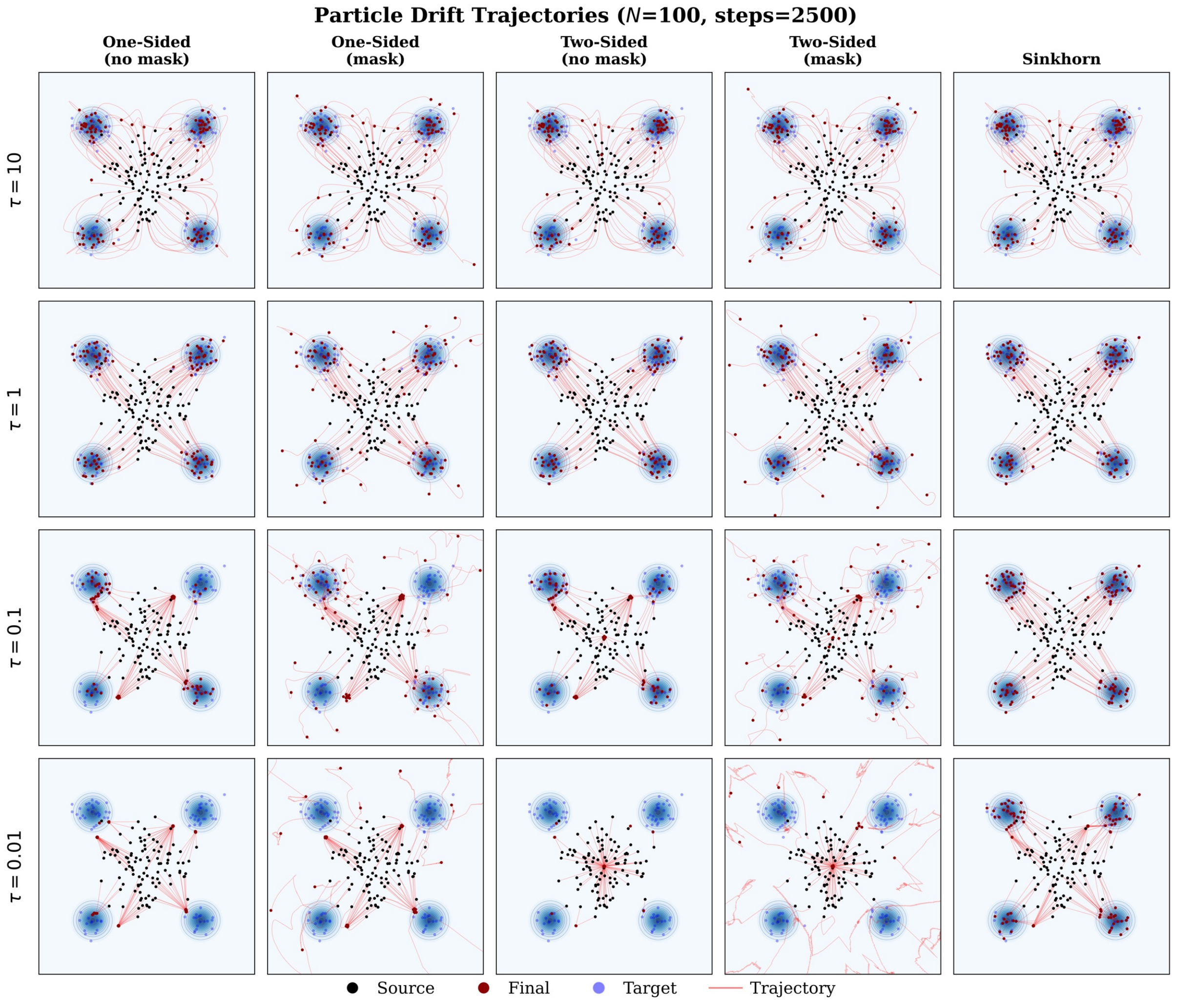}
\caption{
Full drift-trajectory grid for the temperature sweep in Section~\ref{sec:drift_tau}.
Columns correspond to one-sided, two-sided, and Sinkhorn normalization, with and
without self-distance masking where applicable; rows correspond to different
values of $\tau$. Sinkhorn exhibits the most stable trajectories as $\tau$
decreases, whereas one-sided and two-sided normalization become increasingly
sensitive in the low-temperature regime. Empirically, for any choice of $\tau$, the Sinkhorn method admits convergence that is at least as good as (and often better than) the classical one-/two-sided Drift methods.
}
\label{fig:drift_trajectories_full_appendix}
\end{figure}

\section{Additional Results for Toy Experiments}
\label{app:toy}
\raggedbottom
\let\AppendixFloatBarrier\FloatBarrier
\renewcommand{\FloatBarrier}{}
\setlength{\textfloatsep}{5pt plus 1pt minus 1pt}
\setlength{\floatsep}{5pt plus 1pt minus 1pt}
\setlength{\intextsep}{5pt plus 1pt minus 1pt}
\setlength{\abovecaptionskip}{2pt}
\setlength{\belowcaptionskip}{-2pt}
\captionsetup[figure]{skip=2pt}
\captionsetup[table]{skip=2pt}
\setcounter{topnumber}{5}
\setcounter{dbltopnumber}{2}
\setcounter{totalnumber}{8}
\renewcommand{\topfraction}{0.98}
\renewcommand{\dbltopfraction}{0.98}
\renewcommand{\textfraction}{0.02}
\renewcommand{\floatpagefraction}{0.92}
\renewcommand{\dblfloatpagefraction}{0.92}

Figure~\ref{fig:toy_extra}, Figure~\ref{fig:toy_extra_emd_lap} and Figure~\ref{fig:toy_extra_emd} extend Section~\ref{toy_exp} to additional 2D
targets with Gaussian and Laplacian kernels,
$k(x,y)=\exp(-\|x-y\|/\tau)$.
Across 2-Moons, Spiral, 8-Gaussians, and Checkerboard, Sinkhorn consistently
improves mode coverage and convergence, especially for small
$\tau\in\{0.01,0.05,0.1\}$.

\begin{figure}[!b]
\centering
\includegraphics[width=0.94\linewidth]{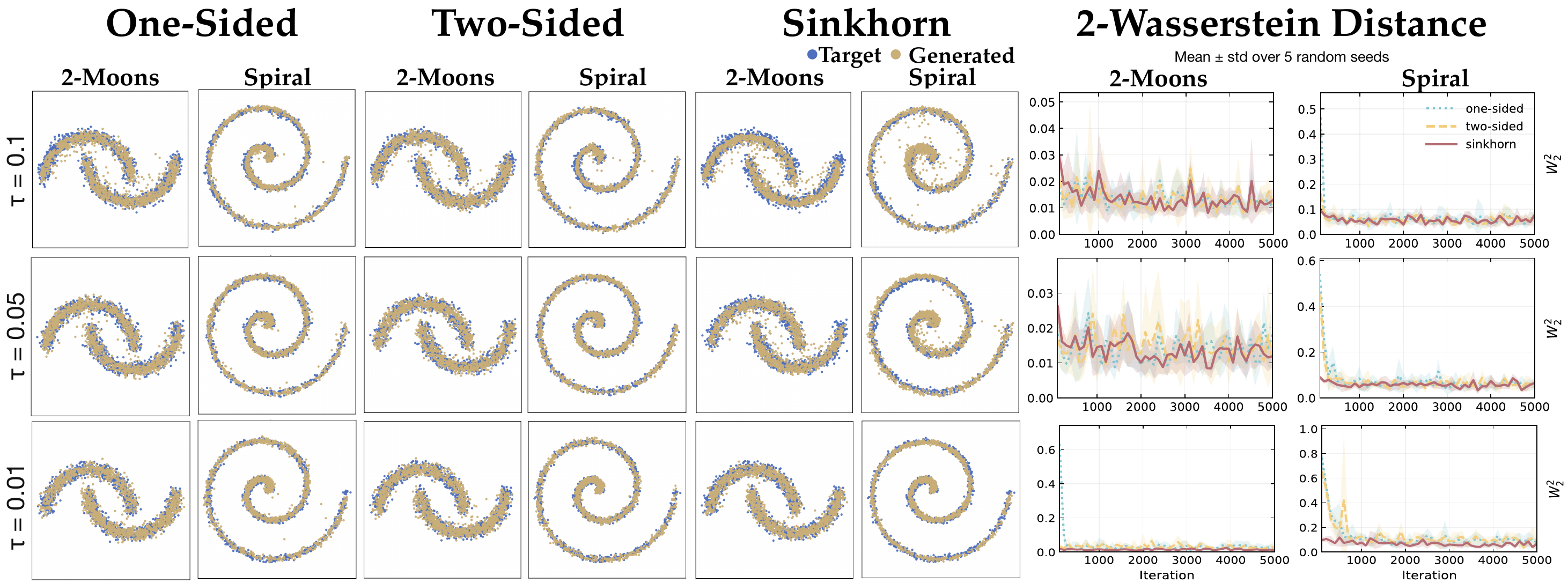}
\caption{2D generative training on 2-Moons and Spiral with a Gaussian kernel across $\tau\in\{0.01,0.05,0.1\}$ and three normalizations (one-sided, two-sided, Sinkhorn). Left: final samples (orange) vs.\ targets (blue); right: $W_2^2$ over $5{,}000$ iterations. Sinkhorn gives better coverage and lower $W_2^2$, especially at small $\tau$.}
\label{fig:toy_extra}
\end{figure}

\begin{figure}[!htbp]
\centering
\includegraphics[width=0.94\linewidth]{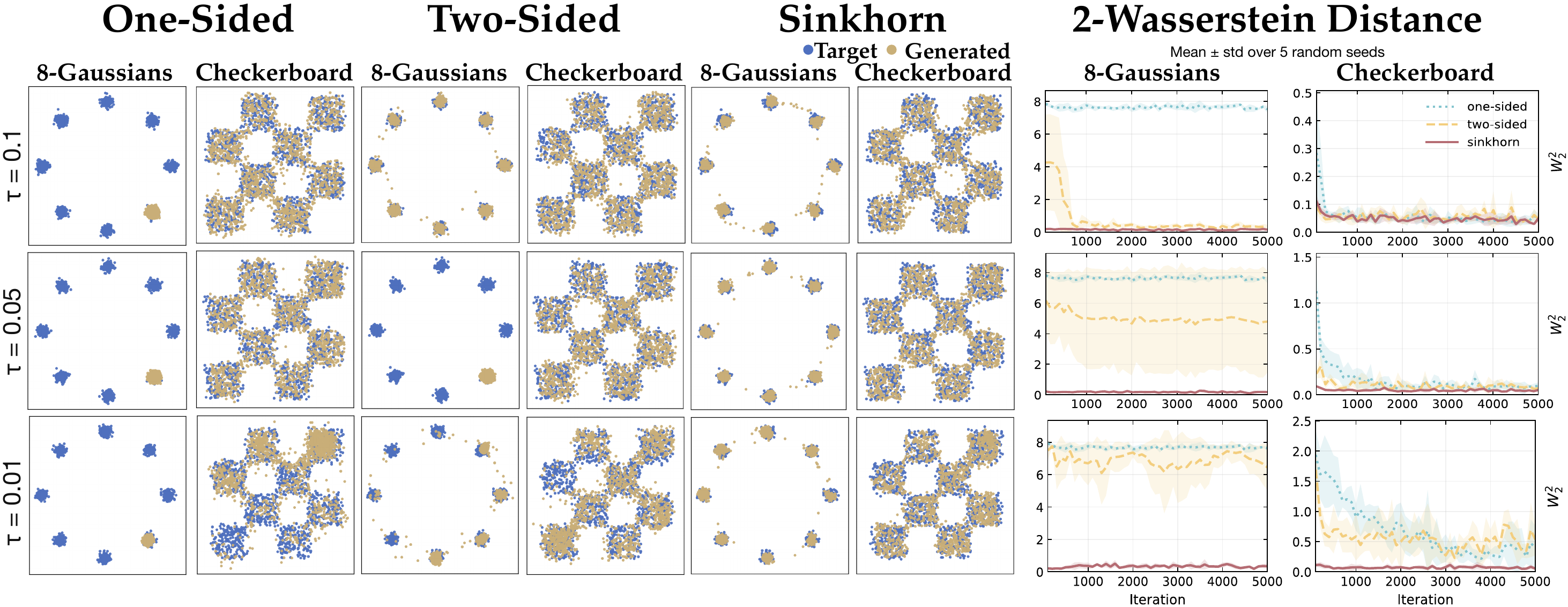}
\caption{2D generative training on 8-Gaussians and Checkerboard with a Laplacian kernel across $\tau\in\{0.01,0.05,0.1\}$ and three normalizations (one-sided, two-sided, Sinkhorn). \textbf{Left}: final samples (orange) vs.\ targets (blue); \textbf{Right}: $W_2^2$ over $5{,}000$ iterations. Sinkhorn gives better coverage and lower $W_2^2$, especially at small $\tau$.}
\label{fig:toy_extra_emd_lap}
\end{figure}

\begin{figure}[!htbp]
\centering
\includegraphics[width=0.94\linewidth]{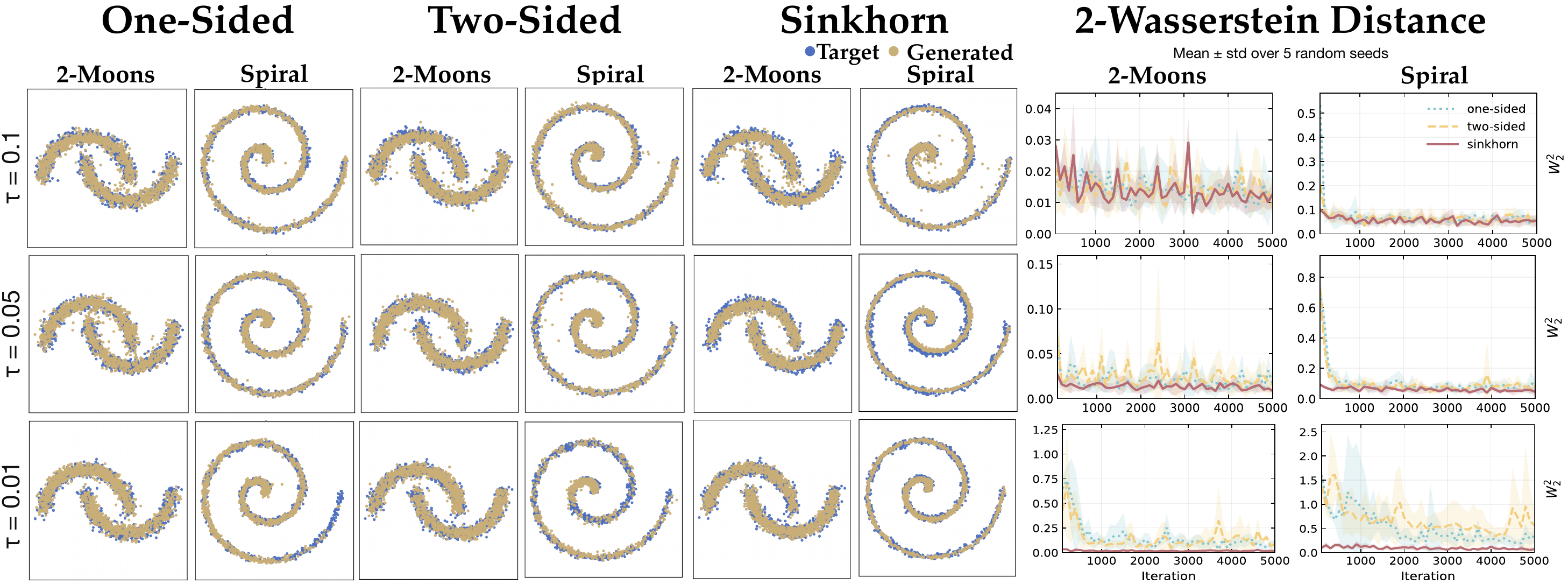}
\caption{2D generative training on 2-Moons and Spiral with a Laplacian kernel across $\tau\in\{0.01,0.05,0.1\}$ and three normalizations (one-sided, two-sided, Sinkhorn). \textbf{Left}: final samples (orange) vs.\ targets (blue); \textbf{Right}: $W_2^2$ over $5{,}000$ iterations. Sinkhorn gives better coverage and lower $W_2^2$, especially at small $\tau$.}
\label{fig:toy_extra_emd}
\end{figure}

\AppendixFloatBarrier
\section{MNIST Temperature Sweep: Laplacian Kernel}
\label{app:mnist_laplacian}

Table~\ref{tab:mnist_laplacian} reports the same $\tau$-sweep experiment as
Section~\ref{sec:mnist} using the Laplacian kernel
$k(x,y)=\exp(-\|x-y\|/\tau)$.
Compared to the Gaussian kernel, the baseline recovers earlier (from $\tau=0.03$
onward) but still completely collapses for $\tau \leq 0.02$.
Sinkhorn remains stable throughout the entire range.

\begin{table}[H]
\centering
\caption{Additional Results for MNIST Experiments.
$^\dagger$Accuracy $\approx 10\%$ indicates mode collapse.
$^\ddagger$Partial collapse.}
\label{tab:mnist_laplacian}
\scriptsize
\setlength{\tabcolsep}{1pt}
\begin{tabular*}{\linewidth}{@{\extracolsep{\fill}}c c c c c @ {\hspace{1pt}\vrule width 0.2pt\hspace{1pt}} c c c c c@{}}
\toprule
& \multicolumn{2}{c}{\textbf{EMD} $\downarrow$} & \multicolumn{2}{c}{\textbf{Accuracy} $\uparrow$}
& & \multicolumn{2}{c}{\textbf{EMD}$\downarrow$} & \multicolumn{2}{c}{\textbf{Accuracy} $\uparrow$} \\
\cmidrule(lr){2-3}\cmidrule(lr){4-5}\cmidrule(lr){7-8}\cmidrule(lr){9-10}
$\tau$ & Baseline & Sinkhorn & Baseline & Sinkhorn
& $\tau$ & Baseline & Sinkhorn & Baseline & Sinkhorn \\
\midrule
0.005 & 73.21 & \textbf{5.40} & $9.97\%^\dagger$ & \textbf{99.97\%}
& 0.030 & \textbf{5.70} & 6.48 & 94.86\% & \textbf{99.98\%} \\
0.010 & 73.21 & \textbf{7.14} & $9.99\%^\dagger$ & \textbf{100.00\%}
& 0.040 & \textbf{4.71} & 6.38 & 96.07\% & \textbf{99.97\%} \\
0.020 & 77.10 & \textbf{6.67} & $10.00\%^\dagger$ & \textbf{100.00\%}
& 0.050 & \textbf{4.46} & 6.34 & 96.17\% & \textbf{99.97\%} \\
0.025 & 12.25 & \textbf{6.55} & $84.21\%^\ddagger$ & \textbf{100.00\%}
& 0.100 & \textbf{4.18} & 6.61 & 96.56\% & \textbf{100.00\%} \\
\bottomrule
\end{tabular*}
\end{table}

\begin{figure}[H]
\centering
\includegraphics[width=\linewidth]{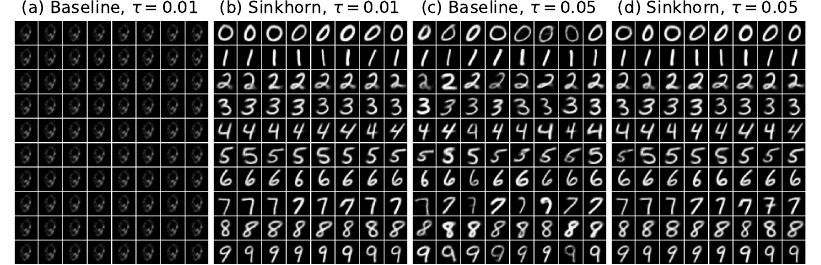}
\caption{Generated MNIST samples (Laplacian kernel).
\textbf{(a)} Baseline collapses at $\tau=0.01$.
\textbf{(b)} Sinkhorn generates all classes correctly at $\tau=0.01$.
\textbf{(c)} Baseline at $\tau=0.05$ shows partial recovery but with visible
class confusion.
\textbf{(d)} Sinkhorn at $\tau=0.05$ remains stable.}
\label{fig:mnist_laplacian}
\end{figure}

\AppendixFloatBarrier

\section{Additional ALAE Qualitative Results}
\label{sec:appendix_alae}

\begin{figure}[H]
  \centering
  \includegraphics[width=\textwidth]{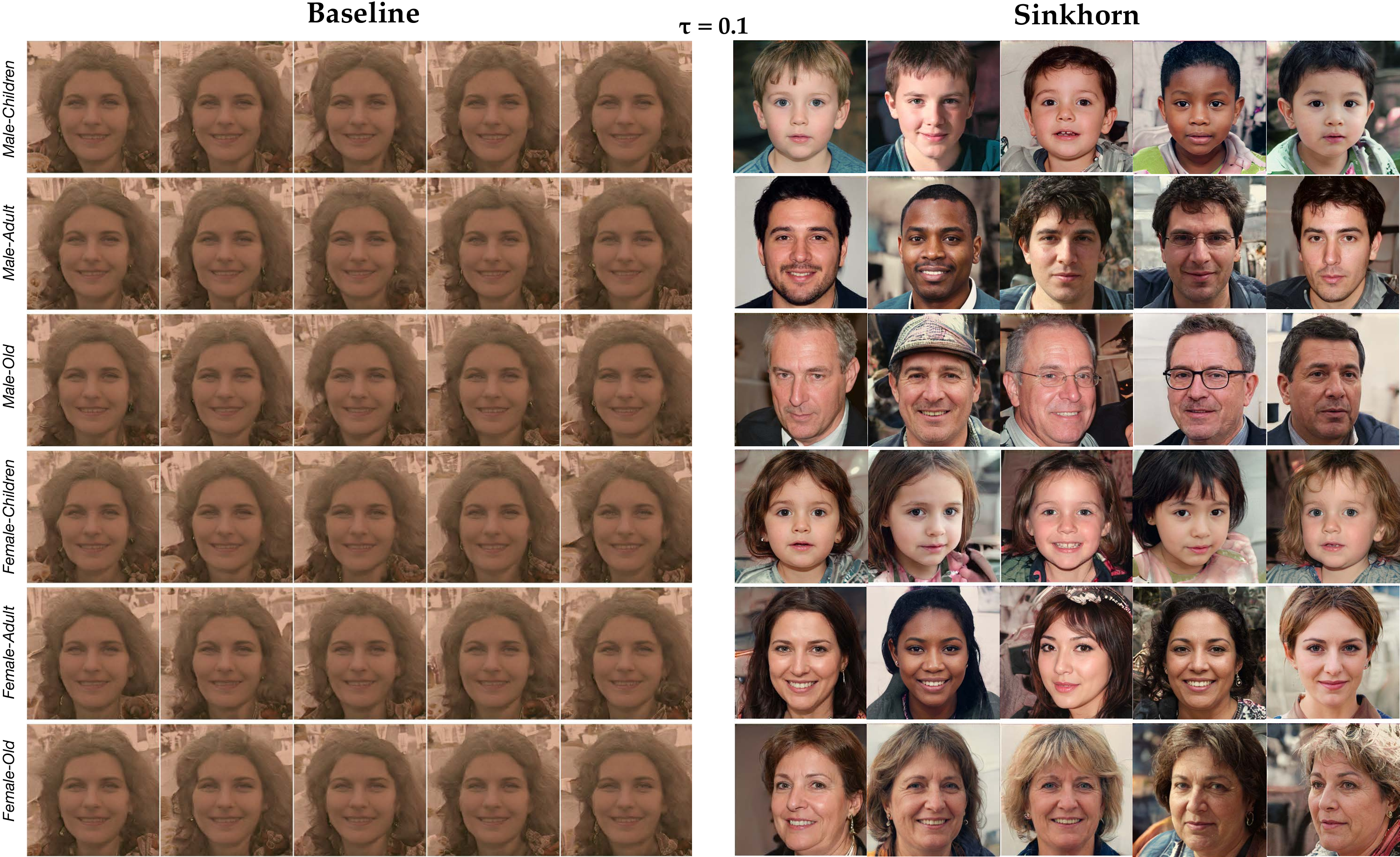}
  \caption{
    Additional qualitative comparison for class-conditional FFHQ generation at $\tau{=}0.1$.
    Each row corresponds to one class; \textbf{Baseline is on the left} and \textbf{Sinkhorn is on the right}.
  }
  \label{fig:alae_appendix_tau01}
\end{figure}

\end{document}